\documentclass[journal]{IEEEtran}

\usepackage[utf8]{inputenc}
\usepackage[english]{babel}
\usepackage{amsthm} 

\usepackage{algorithm}
\usepackage{algpseudocode}

\usepackage{amssymb}  
\usepackage{amsmath}   
\usepackage{graphicx}  
\usepackage{subcaption}  
\usepackage{multirow}  
\usepackage{xcolor}  

\def\b{\ensuremath\boldsymbol}

\usepackage{url}
\newtheorem{definition}{Definition}

\newtheorem{proposition}{Proposition}

\ifCLASSINFOpdf
\else
\fi
\usepackage{eso-pic} 

\begin{document}

\AddToShipoutPictureBG*{%
  \AtPageUpperLeft{%
    \setlength\unitlength{1in}%
    \hspace*{\dimexpr0.5\paperwidth\relax}
    \makebox(0,-0.75)[c]{\normalsize Published in Machine Learning with Applications, Elsevier, Volume 6, Pages 100088, 2021}
    }}

%
\title{Quantile-Quantile Embedding for Distribution Transformation and Manifold Embedding with Ability to Choose the Embedding Distribution}
%
%
%

\author{Benyamin Ghojogh$^1$,
        Fakhri Karray$^2$,
        Mark Crowley$^1$ \\ 
        \hfil \break 
        $^1$Machine Learning Laboratory, Department of ECE, University of Waterloo, Canada \\
        $^2$Centre for Pattern Analysis and Machine Intelligence, Department of ECE, University of Waterloo, Canada \\
        Emails: \{bghojogh, karray, mcrowley\}@uwaterloo.ca
        }

%
%

\markboth{}%
{Shell \MakeLowercase{\textit{et al.}}: Bare Demo of IEEEtran.cls for IEEE Journals}
%



\maketitle

\begin{abstract}
We propose a new embedding method, named Quantile-Quantile Embedding (QQE), for distribution transformation and manifold embedding with the ability to choose the embedding distribution. QQE, which uses the concept of quantile-quantile plot from visual statistical tests, can transform the distribution of data to any theoretical desired distribution or empirical reference sample. Moreover, QQE gives the user a choice of embedding distribution in embedding the manifold of data into the low dimensional embedding space. It can also be used for modifying the embedding distribution of other dimensionality reduction methods, such as PCA, t-SNE, and deep metric learning, for better representation or visualization of data. We propose QQE in both unsupervised and supervised forms. QQE can also transform a distribution to either an exact reference distribution or its shape. We show that QQE allows for better discrimination of classes in some cases. Our experiments on different synthetic and image datasets show the effectiveness of the proposed embedding method.
\end{abstract}


%
\IEEEpeerreviewmaketitle

\section*{Highlights}

\begin{itemize}
\item Transforming distribution of synthetic and image data to reference distributions
\item Embedding classes of digit, facial, and histopathology data to desired distributions
\item Manifold embedding of toy and real data to any desired distributions chosen by user
\item Distribution transformation is evaluated by measures for comparison of distributions
\end{itemize}

\section{Introduction}\label{section_introduction}

Regardless of the data science or machine learning task, there is no doubt that the distribution of the available data instances is highly relevant to the final outcome of any algorithm.
This distribution may be a standard distribution, such as a Gaussian, or it may be some other more exotic distribution or it may be unknown to us entirely.
Now, while pre-processing of data to remove noise, normalize, or otherwise adjust the data to be appropriate for a given algorithm is very common, imagine that we know, or suspect, that the distribution of the given data may not be suitable for the target purpose.
The reason for this could be anything, including knowledge about bias in the datasets, an incompatible distribution type for maximizing class discrimination or for representation and interpretation reasons. 
In these cases, we suggest it would be useful to be able to transform the distribution of dataset, as a pre-processing step, to a known distribution form while maintaining the original local relationships and distances between data instances so that the unique character of the dataset is maintained \cite{saul2003think}. What we present here is a general approach for doing just that, and we include a number of variants and experimental demonstrations of the utility and effectiveness of the approach.

For this distribution transformation, one could try to make all moments of data equal to the moments of the desired distribution \cite{gretton2007kernel,gretton2012kernel} but this quickly becomes computationally expensive. 
Furthermore, moments of non-standard distributions can be hard to compute in some cases. 
Another problem with simply matching all moments is that it results in transformation to the \textit{exact} desired distribution rather than the ``shape'' of the desired distribution which is more desirable. 
Note that transformation of data to the shape of another distribution means that the general shape of the Probability Density Function (PDF) of data becomes similar to the desired PDF regardless of the mean and scale of the distribution. 
One could also imagine, that the desired distribution may only be known indirectly via the distribution of some other empirical reference sample.
Thus, our method for distribution transformation should support providing a desired distribution either as a theoretical PDF/Cumulative Distribution Function (CDF) or as an empirical reference sample. 

A further connection for our proposed approach would be allowing a choice of distribution of embedded data in the field of manifold learning and dimensionality reduction. That is, we might want to choose what distribution the data instances will have after being embedded by a dimensionality reduction method. 
In dimensionality reduction, the choice of embedding distribution is usually not given to the user but it is very relevant as some dimensionality reduction methods already make assumptions about the distribution of neighbors of data points.
Meanwhile, other methods do not even make any assumption on the embedding distribution and yet do not give any choice of embedding distribution to the user. 
We will enumerate some examples for these methods in Section \ref{section_related_work}.

In this paper, we propose a new embedding method, named Quantile-Quantile Embedding (QQE), which can be used for distribution transformation and manifold learning with a user-specified choice of embedding distribution. 
The features and advantages of QQE are summarized as follows:
\begin{enumerate}
\item Distribution transformation to a desired distribution either as a PDF/CDF or an empirical reference distribution given by user. 
The entire dataset can be transformed in an unsupervised manner or every class in the dataset can be transformed in  a supervised manner. 
\item Manifold embedding of high dimensional data into a lower dimensional embedding space with the choice of embedding distribution by the user. 
Manifold embedding in QQE can also modify the embedding of other manifold learning methods, such as Principal Component Analysis (PCA), Fisher Discriminant Analysis (FDA), Student-t distributed Stochastic Neighbor Embedding (t-SNE), Locally Linear Embedding (LLE), and deep metric learning, for better discrimination of classes or better representation/visualization of data. 
\item For both distribution transformation and manifold embedding tasks, the distribution can be transformed to either the exact desired distribution or merely the shape of it. One of the many applications of exact distribution transformation is separation of classes in data. 
\end{enumerate}

The remainder of this paper is organized as follows. We review the related work in Section \ref{section_related_work}. Section \ref{section_qq_plot} introduces the technical background on quantile functions, the univariate quantile-quantile plot, and its multivariate version. In Section \ref{section_qq_embedding}, we propose the QQE method for both distribution transformation and manifold embedding. The experimental results are reported in Section \ref{section_experiments}. Finally, Section \ref{section_conclusion} concludes the paper and enumerates the future directions.

\section{Related Work}\label{section_related_work}

\subsection{Methods for Difference of Distributions}\label{section_related_work_measure_difference_distributions}

An obvious connection exists between this work and the task of computing the difference between two distributions which is a rich field in statistics. 
One of the most well-known methods is the Kullback-Leibler (KL) divergence \cite{kullback1951information}. KL-divergence, which is a relative entropy from one distribution to the other one, has been widely used in deep learning \cite{goodfellow2016deep}.
Another measure for difference of distributions of two random variables is
Maximum Mean Discrepancy (MMD) or kernel two-sample test. It is a measure of difference of two distributions by comparing their moments \cite{gretton2007kernel,gretton2012kernel}. This comparison of moments can be performed after pulling data to the feature space using kernels \cite{hofmann2008kernel}. MMD uses distances in the feature space \cite{scholkopf2001kernel}.
It has been used in machine learning algorithms such as generative moment matching networks \cite{li2015generative,ren2016conditional}. 
Another measure for measuring the relation of two random variables is Hilbert-Schmidt Independence Criterion (HSIC) \cite{gretton2005measuring}. Calculating the dependence of two random variables is difficult while calculating the linear dependence, named correlation, is much simpler. Therefore, for computation of dependence of two random variables, HSIC pulls data to the feature space using kernels \cite{hofmann2008kernel} and then computes the correlation between them in that space. This correlation is a good estimate for the dependence in the input space. Two example uses of HSIC in machine learning are supervised PCA \cite{barshan2011supervised} and supervised guided LLE \cite{alipanahi2011guided}.
Note that the formulas of the three introduced methods for measuring the difference of distributions will be provided in Section \ref{section_Quantitative_Measures}. We have used these measures for quantitatively discussing the results of QQE algorithm. 

\subsection{Quantile Plots for Visual Statistical Tests}

The quantile function for a distribution is defined as the inverse of the CDF \cite{parzen1979nonparametric,hyndman1996sample}. If we plot the quantile function, we will have the quantile plot \cite{galton1885application}. There are multivariate versions of quantile plots \cite{chaudhuri1996geometric} where data instances are multivariate rather than univariate. In case there are two sets of data instances with two distributions, one can match the quantile plots of these two datasets and have the quantile-quantile plot or qq-plot \cite{loy2016variations}. Using the qq-plot, statisticians can visually test whether the two distributions are equal and if not, how different they are \cite{oldford2016self,loy2016variations}. There also exist multivariate versions of qq-plot such as fuzzy qq-plot \cite{easton1990multivariate}. These multivariate qq-plots can be used for visual assessment of whether two distributions match or not. The technical required background on quantile plot and qq-plot are provided in Section \ref{section_qq_plot}. 

\subsection{Embedding Distribution in Manifold Learning Methods}

Some manifold learning and dimensionality reduction methods make an assumption about the distribution of neighbors of data points. For example, Stochastic Neighbor Embedding (SNE) and t-SNE take the Gaussian distribution \cite{hinton2003stochastic} and Cauchy \cite{maaten2008visualizing} (or Student-t \cite{van2009learning}) distribution for the neighborhood of points, respectively. These methods make some strong assumptions about the neighborhood of points and do not give freedom of choice to the user for the embedding distribution. Some manifold learning methods, however, do not even make any assumption about the embedding distribution and yet do not give any choice of embedding distribution to the user. Some examples are PCA \cite{ghojogh2019unsupervised}, Multi-dimensional Scaling (MDS) \cite{cox2008multidimensional}, Sammon mapping \cite{sammon1969nonlinear}, FDA \cite{ghojogh2019fisher}, Isomap \cite{tenenbaum2000global}, LLE \cite{roweis2000nonlinear,saul2003think}, and deep manifold learning \cite{he2016deep,schroff2015facenet}. Note that some of these methods make assumptions but not as a distribution for the embedding. For example, FDA assumes the Gaussian distribution for data in the input space and LLE assumes just unit covariance and zero mean for the embedded data.

\section{Quantile and Quantile-Quantile Plots}\label{section_qq_plot}

\subsection{Quantile Function and Quantile Plot}

The \textit{quantile function} for a distribution is defined as \cite{parzen1979nonparametric,hyndman1996sample}:
\begin{align}\label{equation_quantile_univariate_1}
Q(p) := F^{-1}(p) := \mathrm{inf}\{x~ |~ F(x) \geq p\}, 
\end{align}
where $p \in [0,1]$ is called \textit{position} and $F(x)$ is the CDF.
The quantile function can also be defined as:
\begin{align}\label{equation_quantile_univariate_2}
Q(p) := \arg \min_{\theta \in \mathbb{R}} \mathbb{E}\big[|X - \theta| + (2p-1)(X-\theta)\big],
\end{align}
where $X$ is a random variable with $\mathbb{E}[X] \!<\! \infty$ \cite{ferguson1967mathematical,serfling2004nonparametric}.
The two-dimensional plot $(p, Q(p))$ is called the \textit{quantile plot} which was first proposed by Sir Francis Galton \cite{galton1885application}. Its name was \textit{ogival curve} primarily as it was like an ogive because of the normal distribution of his measured experimental sample. 

If we have a drawn sample, with sample size $n$ from a distribution, the quantile plot is a \textit{sample (or empirical) quantile}. The sample quantile plot is $(p_i, Q(p_i)), \forall i \in \{1, \dots, n\}$.
For the sample quantile, we can determine the $i$-th position, denoted by $p_i$, as:
\begin{align}
p_i := \frac{i - \alpha}{n - \alpha - \beta + 1},
\end{align}
where different values for $\alpha$ and $\beta$ result in different positions \cite{leon1984another}. The simplest type of position is $p_i = i / n$ (with $\alpha=\beta=0$) \cite{parzen1979nonparametric}. The most well-known position is $p_i = (i-0.5)/n$ (with $\alpha=0.5, \beta=0$) \cite{allen1914storage}. However, it is suggested in \cite{hyndman1996sample} to use $p_i = (i-1/3)/(n + 1/3)$ (with $\alpha=\beta=1/3$) which is median unbiased \cite{reiss2012approximate}. 
It is noteworthy that Galton also suggested that we can measure the quantile function only in $p \in \{0.02, 0.09, 0.25, 0.50, 0.75, 0.91, 0.98\}$ as a summary \cite{galton1874proposed}. His summary is promising only for the normal distribution; however, with the power of today's computers we can compute the sample quantile with fine steps.

For the multivariate quantile plot, \textit{spatial rank} fulfills the role played by position in the univariate case.
Spatial rank $\b{u}_i \in \mathbb{R}^d$ of $\b{x}_i \in \mathbb{R}^d$ with respect to the sample $\{\b{x}_j\}_{j=1}^n$ is defined as \cite{mottonen1995multivariate,marden2004positions,serfling2004nonparametric,dhar2014comparison}:
\begin{align}\label{equation_spatial_rank}
\b{u}_i := \frac{1}{n} \sum_{j=1, j \neq i}^n \frac{\b{x}_i - \b{x}_j}{\|\b{x}_i - \b{x}_j\|_2},
\end{align}
whose term in the summation is a generalization of the sign function for the multivariate vector \cite{marden2004positions}. 
Eq. (\ref{equation_quantile_univariate_2}) can be restated as $\arg \min_{\theta} \mathbb{E}(|X - \theta| + u\,(X-\theta))$ where $[-1,1] \ni u := 2p-1$ \cite{chaudhuri1996geometric}. 
The multivariate \textit{spatial quantile} (or \textit{geometrical quantile}) for the multivariate spatial rank $\b{u} \in \mathbb{R}^d$ is defined as: 
\begin{align}
\b{Q}(\b{u}) := \arg \min_{\b{\theta} \in \mathbb{R}^d} \mathbb{E}(\Phi(\b{u}, \b{x}-\b{\theta}) - \Phi(\b{u}, \b{x})),
\end{align}
where $\b{x} \in \mathbb{R}^d$ is a random vector, $\Phi(\b{u},\b{t}) := \|\b{t}\|_2 + \b{u}^\top \b{t}$, and $\b{u}$ is a vector in unit ball, i.e., $\b{u} \in \{\b{v} ~|~ \b{v} \in \mathbb{R}^d, \|\b{v}\|_2 < 1\}$
\cite{chaudhuri1996geometric,serfling2004nonparametric,dhar2014comparison}.

\subsection{Quantile-Quantile Plot}


Assume we have two quantile functions for two univariate distributions. If we match their positions and plot $(Q_1(p), Q_2(p)), \forall p \in [0,1]$, we will have \textit{quantile-quantile plot} or \textit{qq-plot} in short \cite{loy2016variations}. Again, this plot can be an empirical plot, i.e., $(Q_1(p_i), Q_2(p_i)), \forall i \in \{1, \dots, n\}$. Note that the qq-plot is equivalent to the quantile plot for the uniform distribution as we have $Q(p) = p$ in this distribution. Usually, as a statistical test, we want to see whether the first distribution is similar to the second empirical or theoretical distribution \cite{loy2016variations}; therefore, we refer to the first and second distributions as the \textit{observed and reference distributions}, respectively \cite{easton1990multivariate}.
Note that if the qq-plot of two distributions is a line with slope $1$ (angle $\pi/4$) and intercept $0$, the two distributions have the same distributions \cite{oldford2016self}. The slope and the intercept of the line show the difference of spread and location of the two distributions \cite{loy2016variations}. 


In order to extend the qq-plot to multivariate distributions, we can consider the marginal quantiles. However, this fails to take the dependence of marginals into account \cite{dhar2014comparison,easton1990multivariate}. There are several existing methods for a promising generalization. One of these methods is \textit{fuzzy qq-plot} \cite{easton1990multivariate} (note that it is not related to fuzzy logic). In a fuzzy qq-plot, a sample of size $n$ is drawn from the reference distribution and the data points of the two samples are matched using optimization. An affine transformation is also applied to the observed sample in order to have an invariant comparison to the affine transformation. 
In the multivariate qq-plot, the matched data points are used to plot the qq-plots for every component; therefore, we will have $d$ qq-plots where $d$ is the dimensionality of data. Note that these plots are different from the $d$ qq-plots for the marginal distributions.
The technical details of fuzzy qq-plot is explained in the following.

\subsection{Multivariate Fuzzy Quantile-Quantile Plot}

Assume we have a dataset with size $n$ and dimensionality $d$, i.e., $\{\b{x}_i \in \mathbb{R}^d \}_{i=1}^n$. We want to transform its distribution as $\b{x}_i \mapsto \b{y}_i, \forall i \in \{1, \dots, n\}$.
We draw a sample $\{\b{y}_i \in \mathbb{R}^d \}_{i=1}^n$ of size $n$ from the desired (reference) distribution. 
Note that in case we already have a reference sample $\{\b{y}_i \in \mathbb{R}^d \}_{i=1}^m$ rather than the reference distribution, we can employ bootstrapping or oversampling if $m>n$ and $m<n$, respectively, to have $m=n$.
We match the data points $\{\b{x}_i\}_{i=1}^n$ and $\{\b{y}_i\}_{i=1}^n$ \cite{easton1990multivariate}:
\begin{equation}\label{equation_matching_optimization}
\begin{aligned}
& \underset{\b{A},\b{b},\sigma}{\text{minimize}}
& & \sum_{i=1}^n \|\b{x}_i - \b{A}\b{y}_{\sigma(i)} - \b{b}\|_2^2,
\end{aligned}
\end{equation}
where $\b{A} \in \mathbb{R}^{d \times d}$ and $\b{b} \in \mathbb{R}^d$ are used to make the matching problem invariant to affine transformation. If $\mathcal{P}$ is the set of all possible permutations of integers $\{1, \dots, n\}$, we have $\sigma \in \mathcal{P}$. 
This optimization problem finds the best permutation regardless of any affine transformation. 
Note that one can exchange $\b{x}$ and $\b{y}$ in Eq. (\ref{equation_matching_optimization}) and the following equations to match $\b{x}$ with $\b{y}$. As long as matching is performed correctly, that is fine. 

In order to solve this problem, we iteratively switch between solving for $\b{A}$, $\b{b}$, and $\sigma$ until there is no change in $\sigma$ \cite{easton1990multivariate}. Given $\b{A}$ and $\b{b}$, we solve:
\begin{equation}\label{equation_QQE_fuzzy_qqplot_assignment_problem}
\begin{aligned}
&\underset{\sigma}{\text{min.}} \sum_{i=1}^n \|\b{x}_i - \b{A}\b{y}_{\sigma(i)} - \b{b}\|_2^2 \equiv \underset{\b{\Psi}}{\text{min.}} \sum_{i=1}^n \sum_{j=1}^n \b{C}(i,j) \b{\Psi}(i,j),
\end{aligned}
\end{equation}
which is an assignment problem and can be solved using the Hungarian method \cite{kuhn1955hungarian}. $\b{C} \in \mathbb{R}^{n \times n}$ and $\b{\Psi} \in \mathbb{R}^{n \times n}$ are the cost matrix and a matrix with only one $1$ in every row, respectively. Note that $\b{\Psi}(i,j) = 1$ means that $\b{x}_i$ and $\b{y}_j$ are matched. $\b{C}$ should be computed before solving the optimization where $\b{C}(i,j) := \|\b{x}_i - \b{A}\b{y}_{j} - \b{b}\|_2^2$.

According to the $1$'s in the obtained $\b{\Psi}$, we have $\sigma$. Then given $\sigma$, we solve:
\begin{equation}
\begin{aligned}
& \underset{\b{A},\b{b}}{\text{minimize}}
& & \sum_{i=1}^n \|\b{x}_i - \b{A}\b{y}_{\sigma(i)} - \b{b}\|_2^2,
\end{aligned}
\end{equation}
which is a multivariate regression problem. 
The solution is \cite{hastie2009elements}:
\begin{align}
\mathbb{R}^{(d+1) \times d} \ni \b{\beta} := (\breve{\b{Y}}^\top \breve{\b{Y}})^{-1} \breve{\b{Y}}^\top \breve{\b{X}},
\end{align}
where $\mathbb{R}^{n \times (d+1)} \ni \breve{\b{Y}} := \big[[\b{y}_{\sigma(1)}, \dots, \b{y}_{\sigma(n)}]^\top, \b{1}_{n \times 1}\big]$ and $\mathbb{R}^{n \times d} \ni \breve{\b{X}} := \b{X}^\top = [\b{x}_1, \dots, \b{x}_n]^\top$.
We will have $\b{\beta} = [\b{A}, \b{b}]^\top$. Therefore, $\b{A}$ and $\b{b}$ are found where $\b{A}^\top$ is the top $d \times d$ sub-matrix of $\b{\beta}$ and $\b{b}^\top$ is the last row of $\b{\beta}$.

Note that it is better to set the initial rotation matrix to the identity matrix, i.e. $\b{A}^{(0)} = \b{I}$, to reduce the amount of assignment rotation. In this way, only a few iterations will suffice to solve the matching problem. 
This iterative optimization gives us the matching $\sigma$ and the samples $\{\b{x}_i\}_{i=1}^n$ and $\{\b{y}_i\}_{i=1}^n$ are matched. Then, we have $d$ qq-plots, one for every dimension. These qq-plots are named fuzzy qq-plots \cite{easton1990multivariate}.
Considering the spatial ranks, the quantiles are \cite{dhar2014comparison}:
\begin{alignat}{2}
& \b{Q}_X(\b{u}_i) = \b{x}_i,  &&\forall i \in \{1, \dots, n\}, \label{equation_Q_x} \\
& \b{Q}_Y(\b{u}_i) = \b{y}_{\sigma(i)}, ~~~ &&\forall i \in \{1, \dots, n\}. \label{equation_Q_y}
\end{alignat}

\section{Quantile-Quantile Embedding}\label{section_qq_embedding}

In QQE, we want to transform data instances from one dataset $\{\b{x}_i^0\}_{i=1}^n$ to another dataset $\{\b{x}_i\}_{i=1}^n$ where distribution is transformed to a desired target distribution while the local relationships between points are preserved as much as possible.
Formally, we define the task of distribution transformation as follows:
\begin{definition}[distribution transformation]
For a sample $\{\b{x}^0_i\}_{i=1}^n$ of size $n$ in $\mathbb{R}^d$ space, the mapping $\b{x}^0_i \mapsto \b{x}_i, \forall i \in \{1, \dots, n\}$ is a distribution transformation where the distribution of $\{\b{x}_i\}_{i=1}^n$ is the known desired distribution and the local distances of nearby points in $\{\b{x}^0_i\}_{i=1}^n$ are preserved in $\{\b{x}_i\}_{i=1}^n$ as much as possible.
\end{definition}

Distribution transformation can be performed in two ways:
(i) the distribution of data is transformed to the ``exact'' reference distribution, and
(ii) only the ``shape'' of the reference distribution is considered to transform to. 
In the following subsections, we detail these two approaches then we introduce a manifold embedding variation, and finally explain the use of unsupervised and supervised approaches for QQE.  

\subsection{Distribution Transformation to Exact Reference Distribution}\label{section_distribution_transform_exact}

QQE can be used for transformation of data to some exact reference distribution where all moments of the data become equal to the moments of the reference distribution. We start with an initial sample $\{\b{x}^0_i\}_{i=1}^n$ and transform it to $\{\b{x}_i\}_{i=1}^n$ whose distribution is desired to be the same as the distribution of a reference sample $\{\b{y}_{\sigma(i)}\}_{i=1}^n$ or a reference distribution.
For this, we consider the fuzzy qq-plot of $\{\b{x}_i\}_{i=1}^n$ and $\{\b{y}_{\sigma(i)}\}_{i=1}^n$.
When the $d$ qq-plots are obtained by the fuzzy qq-plot, we can use them to embed the data for distribution transformation.
Therefore, the qq-plot of every dimension should be a line with slope one and intercept zero \cite{oldford2016self}. 
Let $Q_l(\b{u}_i) \in \mathbb{R}$ denote the $l$-th dimension of $\mathbb{R}^d \ni \b{Q}(\b{u}_i) = [Q_1(\b{u}_i), \dots, Q_d(\b{u}_i)]^\top$ which is used for the $i$-th data point in the $l$-th qq-plot.
Consider $Q_l(\b{u}_i)$ for the matched data and the reference sample, denoted by $Q_{X, l}(\b{u}_i)$ and $Q_{Y, l}(\b{u}_i)$, respectively. 
In order to have the line in the qq-plot, we should minimize $\sum_{i=1}^n \sum_{l=1}^d \big(Q_{X, l}(\b{u}_i) - Q_{Y, l}(\b{u}_i)\big)^2$.
According to Eqs. (\ref{equation_Q_x}) and (\ref{equation_Q_y}), this cost function is equivalent to $\sum_{i=1}^n \sum_{l=1}^d (x_{i,l} - y_{\sigma(i),l})^2$ where $x_{i,l}$ and $y_{\sigma(i),l}$ denote the $l$-th dimension of $\b{x}_{i} = [x_{i,1}, \dots, x_{i,d}]^\top$ and $\b{y}_{\sigma(i)} = [y_{\sigma(i),1}, \dots, y_{\sigma(i),d}]^\top$, respectively.
In vector form, the cost function is restated as:
\begin{align}
\mathcal{L}_1 := \frac{1}{2}\, \sum_{i=1}^n \|\b{x}_{i} - \b{y}_{\sigma(i)}\|_2^2.
\end{align}

On the other hand, according to our definition of distribution transformation, we should also preserve the local distances of the nearby data points as far as possible to embed data locally \cite{saul2003think}.
For preserving the local distances, we minimize the differences of local distances between data and transformed data. 
We use the $k$-nearest neighbors ($k$-NN) graph for the set $\{\b{x}_i\}_{i=1}^n$. Let $\mathcal{N}_i$ denote the set containing the indices of the $k$ neighbors of $\b{x}_i$.
The cost to be minimized is:
\begin{align}\label{equation_Sammon_cost}
\mathcal{L}_2 := \frac{1}{2a} \sum_{i=1}^n \sum_{j \in \mathcal{N}_i} w_{ij} \big(d_{x}(i,j) - d_x^0(i,j) \big)^2,
\end{align}
where $d_x(i,j) := \|\b{x}_i - \b{x}_j\|_2$, $d_{x}^0(i,j) := \|\b{x}^0_i - \b{x}^0_j\|_2$, and $a := \sum_{i=1}^n \sum_{j \in \mathcal{N}_i} d_x^0(i,j)$ is the normalization factor. The weight $w_{ij} := 1 / d_x^0(i,j)$ gives more value to closer points as expected. 
Note that if $k = n-1$, Eq. (\ref{equation_Sammon_cost}) is the cost function used in Sammon mapping \cite{sammon1969nonlinear,lee2007nonlinear}.
We use this cost as a regularization term in our optimization. Therefore, our optimization problem is:
\begin{equation}
\begin{aligned}
& \underset{\b{X}}{\text{minimize}}
& & \mathcal{L} := \frac{1}{2}\, \sum_{i=1}^n \Big( \|\b{x}_{i} - \b{y}_{\sigma(i)}\|_2^2 \\
& & &~~~~~ + \frac{\lambda}{a} \sum_{j \in \mathcal{N}_i} w_{ij} \big( d_x(i,j) - d_x^0(i,j) \big)^2 \Big),
\end{aligned}
\end{equation}
where $\lambda \! > \! 0$ is the regularization parameter.

\begin{proposition}\label{proposition_QQE_gradient}
The gradient of the cost function with respect to $x_{i,l}$ is: 
\begin{equation}\label{equation_QQE_cost_gradient}
\begin{aligned}
& \frac{\partial \mathcal{L}}{\partial x_{i,l}} = (x_{i,l} - y_{\sigma(i),l}) \\
&~~~~~~~~~ + \frac{\lambda}{a} \sum_{j \in \mathcal{N}_i} \frac{d_x(i,j) - d_x^{0}(i,j)}{d_x(i,j)\, d_x^{0}(i,j)} (x_{i,l} - x_{j,l}).
\end{aligned}
\end{equation}
\end{proposition}

\begin{proof}
Proof in Appendix \ref{section_appendix_A}.
\end{proof}

\begin{proposition}\label{proposition_QQE_second_derivative}
The second derivative of the cost function with respect to $x_{i,l}$ is:
\begin{align}\label{equation_QQE_cost_hessian}
&\frac{\partial^2 \mathcal{L}}{\partial x_{i,l}^2} = 1 + \frac{\lambda}{a} \sum_{j \in \mathcal{N}_i} \Big( \frac{d_x(i,j) - d_x^{0}(i,j)}{d_x(i,j)\, d_x^{0}(i,j)} + \frac{(x_{i,l} - x_{j,l})^2}{\big(d_x(i,j)\big)^3} \Big).
\end{align}
\end{proposition}

\begin{proof}
Proof in Appendix \ref{section_appendix_B}.
\end{proof}

We use the quasi-Newton's method \cite{nocedal2006numerical} for solving this optimization problem inspired by \cite{sammon1969nonlinear}. If we consider the vectors component-wise, the diagonal quasi-Newton's method updates the solution as \cite{lee2007nonlinear}:
\begin{align}
x_{i,l}^{(\nu+1)} := x_{i,l}^{(\nu)} - \eta\, \Big|\frac{\partial^2 \mathcal{L}}{\partial x_{i,l}^2}\Big|^{-1}\, \frac{\partial \mathcal{L}}{\partial x_{i,l}}, 
\end{align}
$\forall i \in \{1, \dots, n\}, \forall l \in \{1, \dots, d\}$, where $\nu$ is the index of iteration, $\eta>0$ is the learning rate, and $|.|$ denotes the absolute value guaranteeing that we move toward the minimum and not maximum in the Newton's method.

\bigbreak
\noindent
\subsection{Distribution Transformation to the Shape of Reference Distribution}\label{section_distribution_transform_shape}

In distribution transformation, we can ignore the location and scale of the reference distribution and merely change the distribution of the observed sample to look like the ``shape'' of the reference distribution regardless of its location and scale. 
In other words, we start with an initial sample $\{\b{x}^0_i\}_{i=1}^n$ and transform it to $\{\b{x}_i\}_{i=1}^n$ whose shape of distribution is desired to be similar to the shape of distribution of a reference sample $\{\b{y}_{\sigma(i)}\}_{i=1}^n$.
Recall that if the qq-plot is a line, the shapes of the distributions are the same where the intercept and slope of the line correspond to the location and scale \cite{oldford2016self}. Therefore, in our optimization, rather than trying to make the qq-plot a line with slope one and intercept zero, we try to make it the closest line possible with any slope and intercept. This line can be found by fitting a line as a least squares problem, i.e., a linear regression problem. For the qq-plot of every dimension, we fit a line to the qq-plot.
If we define $\mathbb{R}^{n} \ni \breve{\b{Q}}_{Y,l} := [Q_{Y,l}(\b{u}_1), \dots, Q_{Y,l}(\b{u}_n)]^\top$, let $\mathbb{R}^{n \times 2} \ni \b{\Gamma}_l := [\b{1}_{n \times 1}, \breve{\b{Q}}_{Y,l}]$. 
Fitting a line to the qq-plot of the $l$-th dimension is the following least squares problem:
\begin{equation}
\begin{aligned}
& \underset{\b{\beta}_l}{\text{minimize}}
& & \frac{1}{2}\, \big\|\b{Q}_{X}(\b{u}_i) - \b{\Gamma}_l\, \b{\beta}_l\big\|_2^2 \overset{(\ref{equation_Q_x})}{=} \frac{1}{2}\, \big\|\b{x}_l - \b{\Gamma}_l\, \b{\beta}_l\big\|_2^2,
\end{aligned}
\end{equation}
whose solution is \cite{hastie2009elements}:
\begin{align}
\mathbb{R}^{2} \ni \b{\beta}_l = (\b{\Gamma}_l^\top \b{\Gamma}_l)^{-1} \b{\Gamma}_l^\top \b{x}_l,
\end{align}
where $\mathbb{R}^n \ni \b{x}_l := [x_{1,l}, \dots, x_{n,l}]^\top$.
The $n$ points on the line fitted to the qq-plot of the $l$-th dimension are:
\begin{align}
\mathbb{R}^{n} \ni \b{\mu}_l := \b{\Gamma}_l\, \b{\beta}_l = [\mu_{{\sigma(1)},l}, \dots, \mu_{{\sigma(n)},l}]^\top,
\end{align}
which are used instead of $\b{Q}_{Y}(\b{u}_i), \forall i$ in our optimization.
Defining $\mathbb{R}^d \ni \breve{\b{\mu}}(\b{y}_{\sigma(i)}) := [\mu_{{\sigma(i)},1}, \dots, \mu_{{\sigma(i)},d}]^\top$, the optimization problem is:
\begin{equation}
\begin{aligned}
& \underset{\b{Y}}{\text{minimize}}
& & \mathcal{L} := \frac{1}{2}\, \sum_{i=1}^n \Big( \|\b{x}_{i} - \breve{\b{\mu}}(\b{y}_{\sigma(i)})\|_2^2 \\
& & &~~~~~ + \frac{\lambda}{a} \sum_{j \in \mathcal{N}_i} w_{ij} \big( d_x(i,j) - d_x^0(i,j) \big)^2 \Big).
\end{aligned}
\end{equation}
Similar to Proposition \ref{proposition_QQE_gradient}, the gradient is:
\begin{equation}
\begin{aligned}
&\frac{\partial \mathcal{L}}{\partial x_{i,l}} = (x_{i,l} - \mu_{{\sigma(i)},l}) \\
&~~~~~~~~~ + \frac{\lambda}{a} \sum_{j \in \mathcal{N}_i} \frac{d_x(i,j) - d_x^{(0)}(i,j)}{d_x(i,j)\, d_x^{(0)}(i,j)} (x_{i,l} - x_{j,l}),
\end{aligned}
\end{equation}
and the second derivative is the same as Proposition \ref{proposition_QQE_second_derivative}. We again solve the optimization using the diagonal quasi-Newton's method \cite{nocedal2006numerical}.

\subsection{Manifold Embedding}\label{section_manifold_embedding}

QQE can be used for manifold embedding in a lower dimensional embedding space where the embedding distribution can be determined by the user. 
As an initialization, the high dimensional data are embedded in a lower dimensional embedding space using a dimensionality reduction method. Thereafter, the low dimensional embedding data are transformed to a desired distribution using QQE.

Any dimensionality reduction method can be utilized for the initialization of data in the low dimensional subspace. Some examples are PCA \cite{ghojogh2019unsupervised} (or classical MDS \cite{cox2008multidimensional}), FDA \cite{ghojogh2019fisher}, Isomap \cite{tenenbaum2000global}, LLE \cite{roweis2000nonlinear}, t-SNE \cite{van2009learning}, and deep features like triplet Siamese features \cite{schroff2015facenet} and ResNet features \cite{he2016deep}. By initialization, an initial embedding of data is obtained in the low dimensional embedding space.

After the initialization, a reference sample is drawn from the reference distribution or is taken from the user. The dimensionality of the reference sample is equal to the dimensionality of the low dimensional embedding space; in other words, the reference sample is in the low dimensional space.
We transform the distribution of the low dimensional data to the reference distribution using QQE. Again, the distribution transformation can be either to the exact or shape of the desired distribution. The proposed methods for distribution transformation to the exact reference distribution or shape of desired distribution were explained in Sections \ref{section_distribution_transform_exact} and \ref{section_distribution_transform_shape} and can be used here for distribution transformation in the low dimensional embedding space.

\subsection{Unsupervised and Supervised Embedding}\label{section_unsupervised_supervised}

QQE, for both tasks of distribution transformation (see Sections \ref{section_distribution_transform_exact} and \ref{section_distribution_transform_shape}) and manifold embedding (see Section \ref{section_manifold_embedding}), can be used in either supervised or unsupervised manners. 
In the following, we explain these two cases:
\begin{itemize}
\item In the \textit{unsupervised} form, all data points are seen together as a cloud of data and the distribution of all data points is transformed to a desired distribution. The unsupervised QQE algorithm for distribution transformation transforms the entire dataset to have the desired distribution. 
For manifold embedding, unsupervised QQE initializes the embedding data into the low dimensional space and then transforms the entire embedded data to have the desired distribution. 
\item In the \textit{supervised} manner, the data points of each class are transformed to have a desired distribution. Hence, in this manner, the user may choose different distributions for each class.
The supervised QQE for distribution transformation transforms the distribution of every class to a desired distribution. 
For manifold learning, supervised QQE initializes the embedding data into the low dimensional space and then transforms the embedded data of every class to a desired distribution. 
Note that QQE for manifold learning can be supervised regardless of whether the dimensionality reduction method used for initialization is unsupervised or supervised.
\end{itemize}

It is noteworthy that in both unsupervised and supervised manners of QQE, the distribution transformation and manifold embedding can be either to the exact reference distribution (see Section \ref{section_distribution_transform_exact}) or to the shape of reference distribution (see Section \ref{section_distribution_transform_shape}).

\section{Experiments}\label{section_experiments}

In this section, we report the experimental results. The code for this paper and its experiments can be found in our Github repository\footnote{https://github.com/bghojogh/Quantile-Quantile-Embedding}. The hardware used for the experiments was Intel Core-i7 CPU with the base frequency 1.80 GHz and  16 GB RAM. 
Table \ref{table_timing} reports the timing of different experiments for giving a sense of pacing in QQE algorithm. Note that the time complexity of QQE algorithm is $\mathcal{O}(n^3 + ndk)$ because of the assignment problem \cite{edmonds1972theoretical} and the optimization steps, respectively. Improvement of time complexity of QQE is a possible future direction discussed in Section \ref{section_conclusion}. 
Note that for all experiments in this article, unless specifically mentioned, we set $\lambda=0.1$, $\eta=0.01$, and $k=10$. A comprehensive discussion on the effect of these hyperparameters will be provided in Section \ref{section_hyperparameter_discussion}.

\begin{figure*}[!t]
\centering
\includegraphics[width=\textwidth]{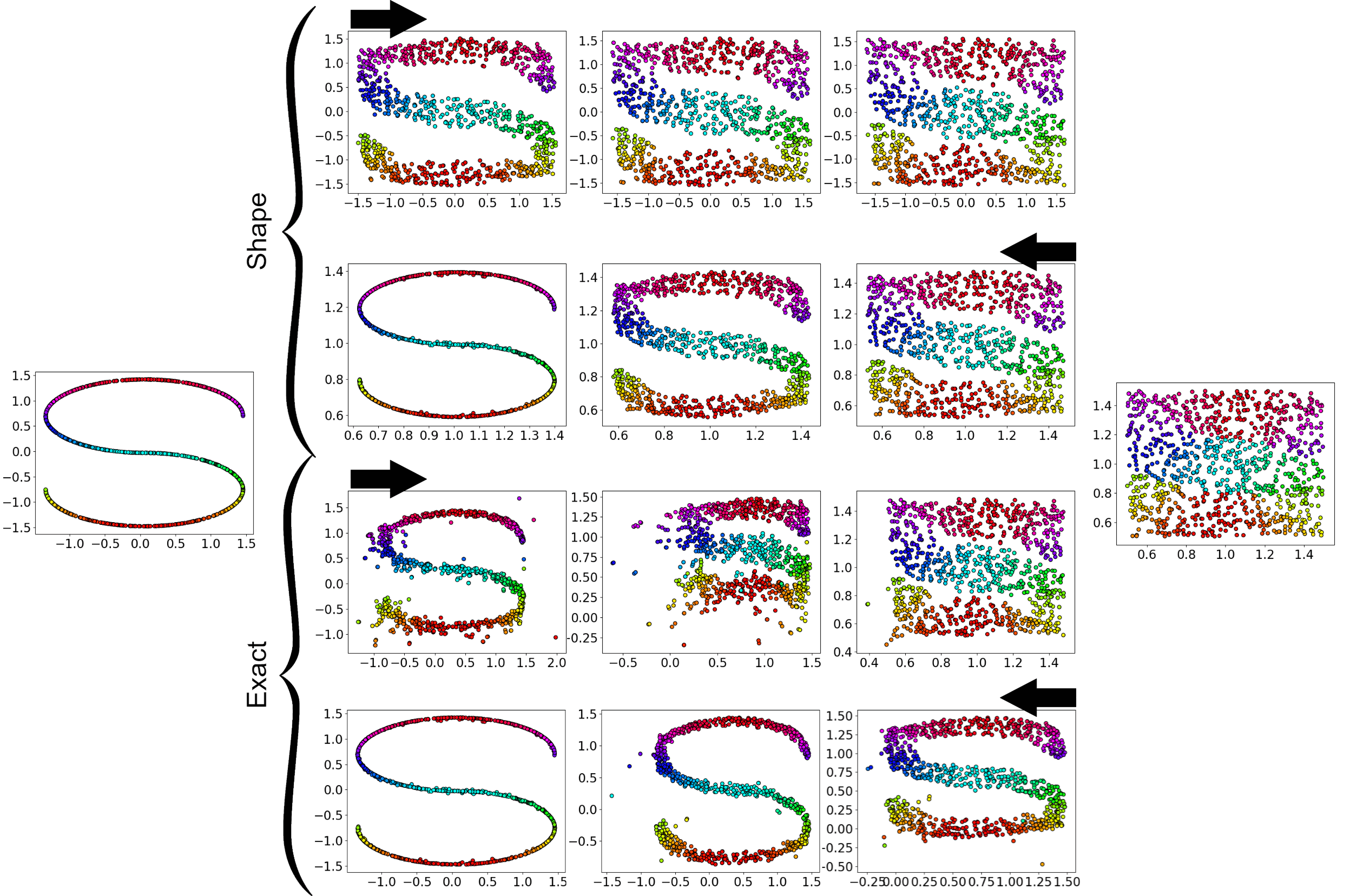}
\caption{Distribution transformation of S-shape and uniform data to each other. The first and second pair of rows correspond to transformation of shape and exact distributions, respectively. The arrows show the direction of gradual changes.}
\label{figure_distribution_transform_synthetic2}
\end{figure*}

\begin{figure*}[!t]
\centering
\includegraphics[width=6.5in]{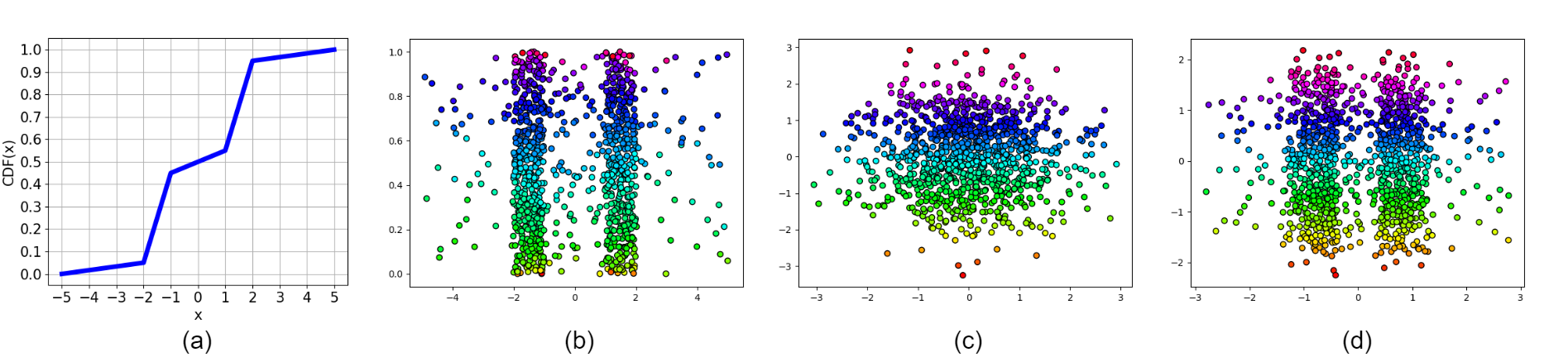}
\caption{Distribution transformation using (a) CDF of reference distribution: (b) the reference data, (c) Gaussian data, and (d) transformed data.}
\label{figure_CDF_synthetic}
\end{figure*}

\begin{figure*}[!t]
\centering
\includegraphics[width=\textwidth]{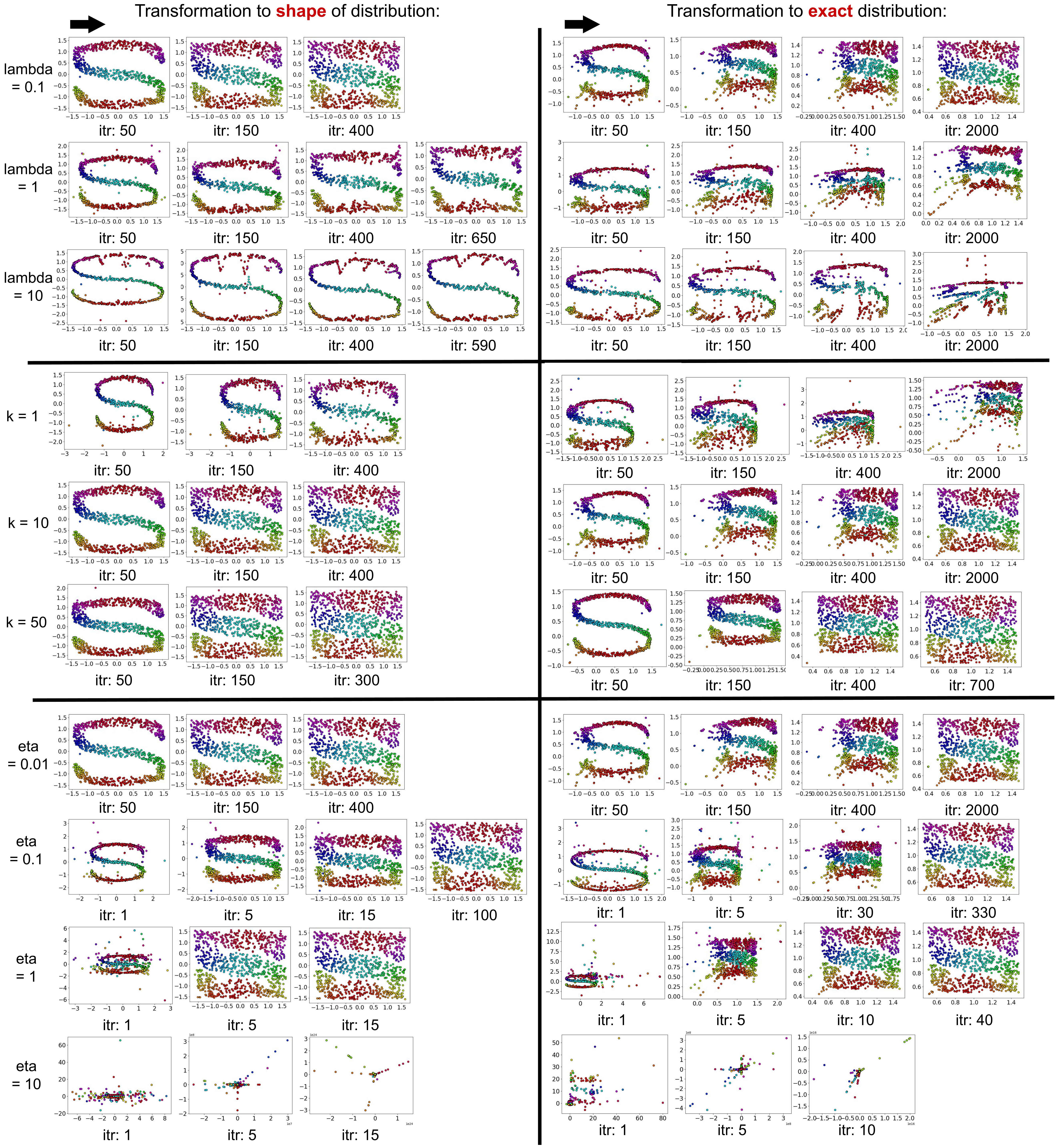}
\caption{Analysis of effect of hyperparameters on the performance of QQE for both transformations to the shape of distribution and exact distribution. The initial and reference distributions are as of the left and right distributions depicted in Fig. \ref{figure_distribution_transform_synthetic2}. }
\label{figure_hyperparameters}
\end{figure*}

\begin{table*}
\caption{The runtime for experiments in this paper. The reported times sum the timings for matching and fuzzy qq-plot iterations. In QQE manifold embedding, the time for initialization is not included. All times are in seconds. Letters U and S denote unsupervised and supervised QQE approaches, respectively.}
\label{table_timing}
\begin{minipage}{\textwidth}
\renewcommand{\arraystretch}{1.3}  
\centering
\scalebox{0.8}{    
\begin{tabular}{l || c | c | c | c | c | c | c | c }
\hline
\hline
Experiment & Fig. \ref{figure_distribution_transform_synthetic2} (1st row) & Fig. \ref{figure_distribution_transform_synthetic2} (2nd row) & Fig. \ref{figure_distribution_transform_synthetic2} (3rd row) & Fig. \ref{figure_distribution_transform_synthetic2} (4th row) & Fig. \ref{figure_CDF_synthetic} & Fig. \ref{figure_distribution_transform_Glasses}  \\
\hline
Time & 206.56 & 253.49 & 364.08 & 278.76 & 51.57 & 1064.85  \\
\hline
\hline
Experiment & Fig. \ref{figure_manifoldEmbedding_synthetic} (PCA, U) & Fig. \ref{figure_manifoldEmbedding_synthetic} (PCA, S) & Fig. \ref{figure_manifoldEmbedding_synthetic} (FDA, U) & Fig. \ref{figure_manifoldEmbedding_synthetic} (FDA, S) & Fig. \ref{figure_manifoldEmbedding_synthetic} (Isomap, U) & Fig. \ref{figure_manifoldEmbedding_synthetic} (Isomap, S) & Fig. \ref{figure_manifoldEmbedding_synthetic} (LLE, U) & Fig. \ref{figure_manifoldEmbedding_synthetic} (LLE, S) \\
\hline
Time & 206.81 & 450.87 & 262.82 & 269.86 & 187.53 & 283.65 & 328.74 & 276.10  \\
\hline
\hline
Experiment & Fig. \ref{figure_manifoldEmbedding_synthetic} (t-SNE, U) & Fig. \ref{figure_manifoldEmbedding_synthetic} (t-SNE, S) & Fig. \ref{figure_manifoldEmbedding_synthetic} (S, Exact)  \\
\hline
Time & 289.68 & 289.62 & 365.24  \\
\hline
\hline
Experiment & Fig. \ref{figure_manifoldEmbedding_mnist} (PCA, U) & Fig. \ref{figure_manifoldEmbedding_mnist} (PCA, S) & Fig. \ref{figure_manifoldEmbedding_mnist} (FDA, U) & Fig. \ref{figure_manifoldEmbedding_mnist} (FDA, S) & Fig. \ref{figure_manifoldEmbedding_mnist} (Isomap, U) & Fig. \ref{figure_manifoldEmbedding_mnist} (Isomap, S) & Fig. \ref{figure_manifoldEmbedding_mnist} (LLE, U) & Fig. \ref{figure_manifoldEmbedding_mnist} (LLE, S) \\
\hline
Time & 2787.25 & 2527.27 & 6803.98 & 2660.45 & 6654.53 & 14490.06 & 4885.59 & 649.08 \\
\hline
\hline
Experiment & Fig. \ref{figure_manifoldEmbedding_mnist} (t-SNE, U) & Fig. \ref{figure_manifoldEmbedding_mnist} (t-SNE, S) & Fig. \ref{figure_manifoldEmbedding_mnist} (ResNet, U) & Fig. \ref{figure_manifoldEmbedding_mnist} (ResNet, S) & Fig. \ref{figure_manifoldEmbedding_mnist} (Siamese, U) & Fig. \ref{figure_manifoldEmbedding_mnist} (Siamese, S) & Fig. \ref{figure_manifoldEmbedding_mnist} (S, Exact) &  \\
\hline
Time & 5373.26 & 2644.21 & 6075.75 & 2332.44 & 6281.83 & 30564.37 & 4893.62 \\
\hline
\hline
Experiment & Fig. \ref{figure_class_separation} (synthetic) & Fig. \ref{figure_class_separation} (face) & Fig. \ref{figure_Histopathology} (1st row) & Fig. \ref{figure_Histopathology} (2nd row) & Fig. \ref{figure_Histopathology} (3rd row)  \\
\hline
Time & 365.95 & 4814.96 & 1181.47 & 1488.40 & 1597.77 \\
\hline
\hline
\end{tabular}%
}
\end{minipage}
\end{table*}

\begin{table*}
\caption{The quantitative evaluation of QQE embeddings for experiments in this paper. Letters U and S denote unsupervised and supervised QQE approaches, respectively. For supervised cases, the reported number is the average of that measure among classes. In every cell of table, the left-side and right-side numbers correspond to before and after applying QQE algorithm, respectively.}
\label{table_comparison_distributions}
\begin{minipage}{\textwidth}
\renewcommand{\arraystretch}{1.3}  
\centering
\scalebox{0.7}{    
\begin{tabular}{l || c | c | c | c | c | c | c | c }
\hline
\hline
Experiment & Fig. \ref{figure_distribution_transform_synthetic2} (1st row) & Fig. \ref{figure_distribution_transform_synthetic2} (2nd row) & Fig. \ref{figure_distribution_transform_synthetic2} (3rd row) & Fig. \ref{figure_distribution_transform_synthetic2} (4th row) & Fig. \ref{figure_CDF_synthetic} & Fig. \ref{figure_distribution_transform_Glasses}  \\
\hline
KL-divergence & 4.90E-2 | 3.58E-2 & 3.91E-2 | 3.70E-2 & 4.90E-2 | 4.04E-2 & 3.91E-2 | 4.39E-2 & 3.18E-1 | 2.56E-1 & 2.40E-3 | 1.23E-3 \\
MMD$^2$ & 5.99E-1 | 5.84E-1 & 2.22E-16 | 6.07E-5  & 5.99E-1 | 5.47E-5 & 2.22E-16 | 3.86E-1 & 1.92E-1 | 1.87E-1 & 1.68E-2 | 1.68E-2 \\
HSIC & 7.33E-5 | 8.11E-5 & 2.11E-5 | 1.73E-5 & 7.33E-5 | 1.95E-5 & 2.11E-5 | 5.68E-5 & 3.18E-4 | 3.25E-4 & 8.47E-3 | 8.47E-3 \\
\hline
\hline
Experiment & Fig. \ref{figure_manifoldEmbedding_synthetic} (PCA, U) & Fig. \ref{figure_manifoldEmbedding_synthetic} (PCA, S) & Fig. \ref{figure_manifoldEmbedding_synthetic} (FDA, U) & Fig. \ref{figure_manifoldEmbedding_synthetic} (FDA, S) & Fig. \ref{figure_manifoldEmbedding_synthetic} (Isomap, U) & Fig. \ref{figure_manifoldEmbedding_synthetic} (Isomap, S) & Fig. \ref{figure_manifoldEmbedding_synthetic} (LLE, U) & Fig. \ref{figure_manifoldEmbedding_synthetic} (LLE, S) \\
\hline
KL-divergence & 2.10E-1 | 1.05E-1 & 9.23E-2 | 6.33E-3 & 1.30E-1 | 7.78E-2 & 8.22E-2 | 1.42E-2 & 2.52E-1 | 1.20E-1 & 8.45E-2 | 1.80E-2 & 3.26E-1 | 1.83E-1 & 1.45E-1 | 7.10E-2 \\
MMD$^2$ & 6.48E-1 | 5.92E-1 & 7.46E-1 | 7.46E-1 & 6.64E-1 | 5.93E-1 & 9.76E-1 | 9.77E-1 & 6.86E-1 | 6.21E-1 & 7.85E-1 | 7.88E-1 & 6.29E-1 | 7.60E-1 & 8.12E-1 | 7.79E-1 \\
HSIC & 7.52E-5 | 8.89E-5 & 1.82E-2 | 2.22E-2 & 1.26E-4 | 1.50E-4 & 1.65E-2 | 2.03E-2 & 9.37E-5 | 9.20E-5 & 1.62E-2 | 1.91E-2 & 1.28E-4 | 1.62E-4 & 6.46E-3 | 5.53E-3 \\
\hline
\hline
Experiment & Fig. \ref{figure_manifoldEmbedding_synthetic} (t-SNE, U) & Fig. \ref{figure_manifoldEmbedding_synthetic} (t-SNE, S) & Fig. \ref{figure_manifoldEmbedding_synthetic} (S, Exact)  \\
\hline
KL-divergence & 5.23E-2 | 5.43E-2 & 7.92E-2 | 2.43E-2 & 7.98E-2 | 4.34E-2 \\
MMD$^2$ & 8.59E-1 | 8.57E-1 & 8.65E-1 | 8.62E-1 & 7.59E-1 | 2.55E-1  \\
HSIC & 1.43E-4 | 1.44E-4 & 1.05E-3 | 7.14E-4 & 1.76E-2 | 1.66E-2 \\
\hline
\hline
Experiment & Fig. \ref{figure_manifoldEmbedding_mnist} (PCA, U) & Fig. \ref{figure_manifoldEmbedding_mnist} (PCA, S) & Fig. \ref{figure_manifoldEmbedding_mnist} (FDA, U) & Fig. \ref{figure_manifoldEmbedding_mnist} (FDA, S) & Fig. \ref{figure_manifoldEmbedding_mnist} (Isomap, U) & Fig. \ref{figure_manifoldEmbedding_mnist} (Isomap, S) & Fig. \ref{figure_manifoldEmbedding_mnist} (LLE, U) & Fig. \ref{figure_manifoldEmbedding_mnist} (LLE, S) \\
\hline
KL-divergence & 2.66E-1 | 1.14E-1 & 1.56E-1 | 3.77E-2 & 1.93E-1 | 8.36E-2 & 1.61E-1 | 4.41E-2 & 2.42E-1 | 1.02E-1 & 1.58E-1 | 3.85E-2 & 6.50E-1 | 5.51E-1 & 1.50E-1 | 1.19E-1 \\
MMD$^2$ & 4.48E-1 | 4.81E-1 & 5.63E-1 | 5.46E-1 & 4.20E-1 | 4.11E-1 & 7.81E-1 | 7.78E-1 & 8.53E-1 | 8.53E-1 & 5.43E-1 | 5.45E-1 & 3.30E-1 | 2.06E-1 & 1.23E0 | 1.24E0 \\
HSIC & 4.62E-5 | 4.70E-5 & 1.45E-2 | 1.91E-2 & 3.11E-5 | 3.16E-5 & 4.46E-2 | 6.01E-2 & 1.47E-5 | 1.47E-5 & 1.95E-3 | 2.49E-3 & 5.54E-5 | 6.52E-5 & 1.40E-2 | 1.64E-2 \\
\hline
\hline
Experiment & Fig. \ref{figure_manifoldEmbedding_mnist} (t-SNE, U) & Fig. \ref{figure_manifoldEmbedding_mnist} (t-SNE, S) & Fig. \ref{figure_manifoldEmbedding_mnist} (ResNet, U) & Fig. \ref{figure_manifoldEmbedding_mnist} (ResNet, S) & Fig. \ref{figure_manifoldEmbedding_mnist} (Siamese, U) & Fig. \ref{figure_manifoldEmbedding_mnist} (Siamese, S) & Fig. \ref{figure_manifoldEmbedding_mnist} (S, Exact) &  \\
\hline
KL-divergence & 5.11E-2 | 4.42E-2 & 1.10E-1 | 4.22E-2 & 1.89E-1 | 8.83E-2 & 1.65E-1 | 3.88E-2 & 1.02E-1 | 5.66E-2 & 1.31E-1 | 3.16E-2 & 1.48E-1 | 1.03E-1 & \\
MMD$^2$ & 8.13E-1 | 8.12E-1 & 5.49E-1 | 5.48E-1 & 6.15E-1 | 6.34E-1 & 6.72E-1 | 6.58E-1 & 4.87E-2 | 4.85E-2 & 1.23EE0 | 1.24E0 & 5.61E-1 | 4.66E-1 & \\
HSIC & 1.87E-5 | 1.85E-5 & 2.62E-3 | 3.24E-3 & 2.28E-5 | 2.32E-5 & 4.29E-2 | 5.79E-2 & 5.30E-5 | 5.49E-5 & 2.10E-2 | 2.31E-2 & 1.46E-2 | 3.65E-2 & \\
\hline
\hline
Experiment & Fig. \ref{figure_class_separation} (synthetic) & Fig. \ref{figure_class_separation} (face) & Fig. \ref{figure_Histopathology} (1st row) & Fig. \ref{figure_Histopathology} (2nd row) & Fig. \ref{figure_Histopathology} (3rd row) \\
\hline
KL-divergence & 2.00E-2 | 3.13E-2 & 1.27E-3 | 1.11E-3 & 3.35E-12 | 1.54E-14 & 4.33E-15 | 3.29E-15 & 3.33E-15 | 2.93E-15 \\
MMD$^2$ & 8.54E-1 | 3.02E-1 & 1.23E-2 | 1.23E-2 & 3.61E-1 | 3.04E-1 & 3.33E-1 | 3.13E-1 & 3.10E-1 | 1.04E-1 \\
HSIC & 4.33E-2 | 4.80E-2 & 6.02E-3 | 6.02E-3 & 2.53E-5 | 6.72E-5 & 1.36E-4 | 3.25E-4 & 3.21E-4 | 5.46E-4 \\
\hline
\hline
\end{tabular}%
}
\end{minipage}
\end{table*}



\subsection{Quantitative Measures Used for Difference of Distributions}\label{section_Quantitative_Measures}

In our experimental results, in addition to illustrating the visualization of distribution transformation either in the input space or in the embedding space, we report several quantitative measurements for validating distribution transformation theoretically. 
Table \ref{table_comparison_distributions} reports the quantitative measurements for all experiments, 
showing the improvement of change of distributions to the desired distributions using the QQE algorithm. 
The three measures used are KL-divergence, MMD and HSIC which are briefly defined below. 

Assume we have two samples from the following distributions: $\{\b{x}_i\}_{i=1}^n \sim \mathcal{P}$ and $\{\b{y}_i\}_{i=1}^n \sim \mathcal{Q}$.
The first used measure is KL-divergence \cite{kullback1951information}. The KL-divergence for discrete samples is defined as:
\begin{align}
\text{KL}(\mathcal{P} \| \mathcal{Q}) := \sum_{i=1}^n \mathcal{P}(\b{x}_i) \log\Big(\frac{\mathcal{P}(\b{x}_i)}{\mathcal{Q}(\b{y}_i)}\Big), 
\end{align}
for the difference of distributions $\mathcal{P}$ and $\mathcal{Q}$. We estimate $\mathcal{P}(\b{x}_i)$ and $\mathcal{Q}(\b{y}_i)$ using kernel density estimation with Gaussian kernels and the Scott’s rule \cite{scott2015multivariate}. 
Note that KL $\geq 0$ where KL $=0$ means the two distributions are equivalent. 
After applying QQE for distribution transformation or manifold embedding, we would expect the KL-divergence between the sample $\{\b{x}_i\}_{i=1}^n$ and the reference sample $\{\b{y}_i\}_{i=1}^n$ to be reduced.
Note that the amount of reduction of KL-divergence is not necessarily meaningful as KL-divergence does not have any upperbound. 

The second measure used for difference of distributions is MMD \cite{gretton2007kernel,gretton2012kernel}. It compares the moments of distributions using distances in the feature space \cite{scholkopf2001kernel}. 
Let $\b{\phi}(\b{x})$ be the pulling function from the input to the feature space and $k(\b{x}_i, \b{x}_j) := \b{\phi}(\b{x}_i)^\top \b{\phi}(\b{x}_j)$ be the kernel function \cite{hofmann2008kernel}.
It is defined as:
\begin{align}
\text{MMD}^2(\mathcal{P}, \mathcal{Q}) &:= \Big\|\frac{1}{n} \sum_{i=1}^n \b{\phi}(\b{x}_i) - \frac{1}{n} \sum_{i=1}^n \b{\phi}(\b{y}_i)\Big\|^2 \nonumber \\
&= \frac{1}{n^2} \sum_{i=1}^n \sum_{j=1}^n k(\b{x}_i, \b{x}_j) + \frac{1}{n^2} \sum_{i=1}^n \sum_{j=1}^n k(\b{y}_i, \b{y}_j) \nonumber \\
&~~~~ - \frac{2}{n^2} \sum_{i=1}^n \sum_{j=1}^n k(\b{x}_i, \b{y}_j),
\end{align}
where $\|.\|$ denotes a norm in the feature/Hilbert space. 
Note that MMD $\geq 0$ where MMD $=0$ means the two distributions are equivalent. 
After applying QQE for distribution transformation or manifold embedding, it is mostly expected to have smaller MMD between the sample $\{\b{x}_i\}_{i=1}^n$ and the reference sample $\{\b{y}_i\}_{i=1}^n$.
As MMD does not have any upperbound, the amount of reduction of MMD is not important but the reduction itself is mostly expected. 

The third method used in this paper for quantitative measurements is HSIC \cite{gretton2005measuring}. It estimates the dependence of two random variables by computing the correlation of the pulled data $\b{\phi}(\b{x}_i)$ and $\b{\phi}(\b{x}_j)$ using the Hilbert-Schmidt norm of their cross-covariance. One can refer to \cite{gubner2006probability} for definitions of the Hilbert-Schmidt norm and the cross-covariance matrix of two random variables. An empirical estimate of HSIC between samples $\{\b{x}_i\}_{i=1}^n$ and $\{\b{y}_i\}_{i=1}^n$ is \cite{gretton2005measuring}:
\begin{align}
\text{HSIC}(X, Y) := \frac{1}{(n-1)^2} \textbf{tr}(\b{K}_x \b{H} \b{K}_y \b{H}),
\end{align}
where $\textbf{tr}(.)$ denotes the trace of matrix and $\b{K}_x$ and $\b{K}_y$ are kernels over samples $\{\b{x}_i\}_{i=1}^n$ and $\{\b{y}_i\}_{i=1}^n$, respectively \cite{hofmann2008kernel}. $\mathbb{R}^{n \times n} \ni \b{H} := \b{I} - (1/n) \b{1}\b{1}^\top$ is the centering matrix where $\b{I}$ and $\b{1}$ denote the identity matrix and the vector of ones, respectively. 
Note that HSIC $\geq 0$ where HSIC $=0$ means the two random variables are independent. The more the HSIC is, the more dependent the variables are. After applying QQE for distribution transformation or manifold embedding, it is mostly expected to have larger HSIC between the sample $\{\b{x}_i\}_{i=1}^n$ and the reference sample $\{\b{y}_i\}_{i=1}^n$.
As HSIC does not have any upperbound, the amount of increase of HSIC is not important but the increase itself is mostly expected. 
It is noteworthy that the trend of decrease in KL-divergence and MMD often coincide with the trend of increase in HSIC; although in some rare cases, this coincident does not hold. 

\subsection{Distribution Transformation for Synthetic Data}

To visually show how distribution transformation works, we report the results of QQE on some synthetic datasets. In the following, we report several different possible cases for distribution transformation.

\begin{figure}[!t]
\centering
\includegraphics[width=3.45in]{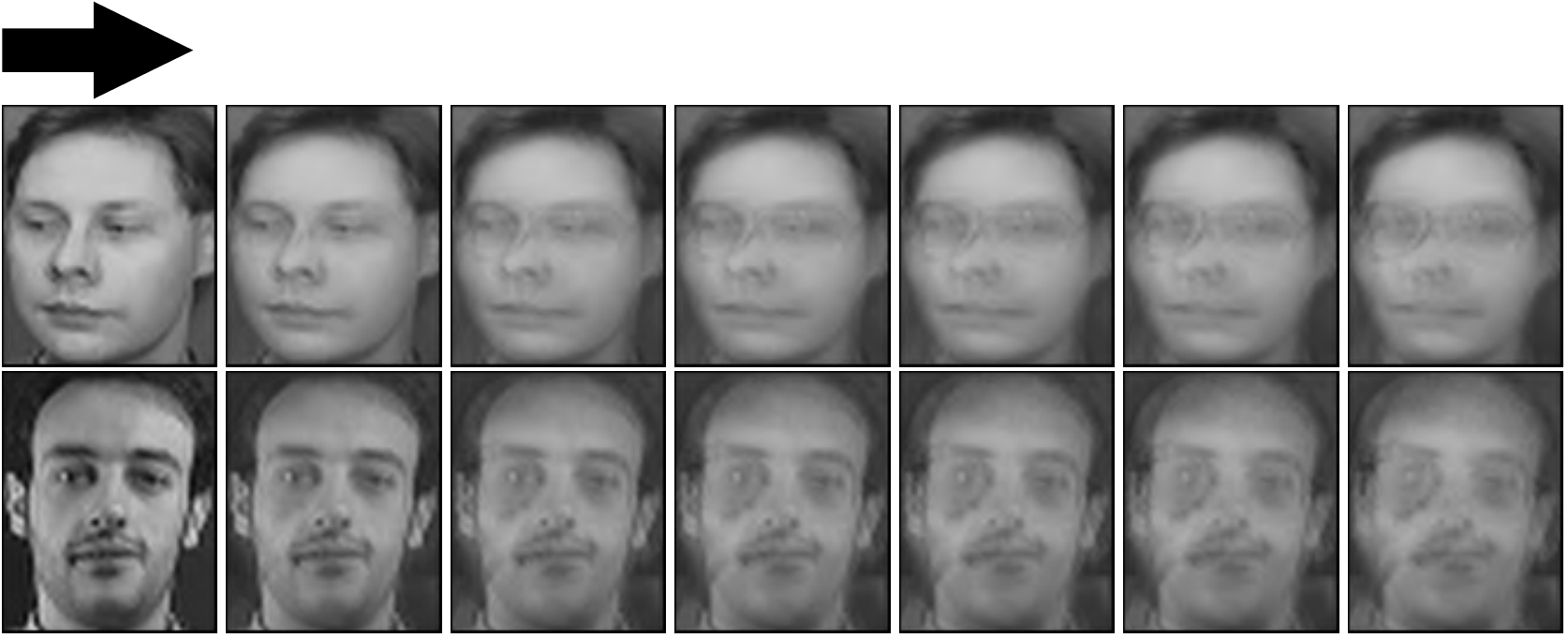}
\caption{Distribution transformation of facial images \cite{samaria1994parameterisation,web_ORL_dataset} without eyeglasses to the shape of images with eyeglasses. The arrow shows the direction of gradual changes.}
\label{figure_distribution_transform_Glasses}
\end{figure}

\begin{figure*}[!t]
\centering
\includegraphics[width=6.5in]{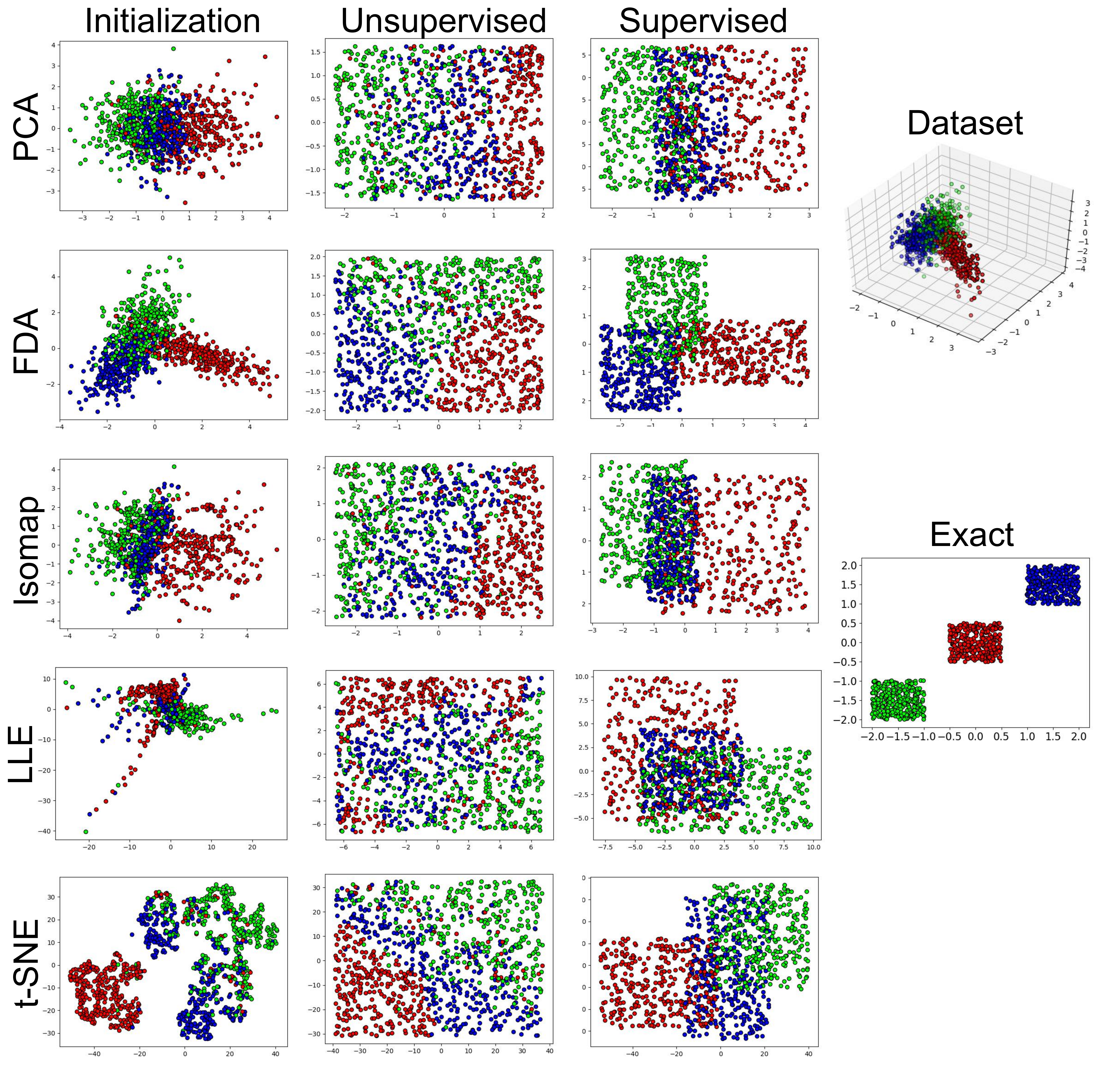}
\caption{Unsupervised and supervised exact manifold embedding of the synthetic data with different initializations. Transformation to exact reference distribution is also shown. The initialization of LLE is scaled by constant to be in range of other embeddings.}
\label{figure_manifoldEmbedding_synthetic}
\end{figure*}

\begin{figure*}[!t]
\centering
\includegraphics[width=6.4in]{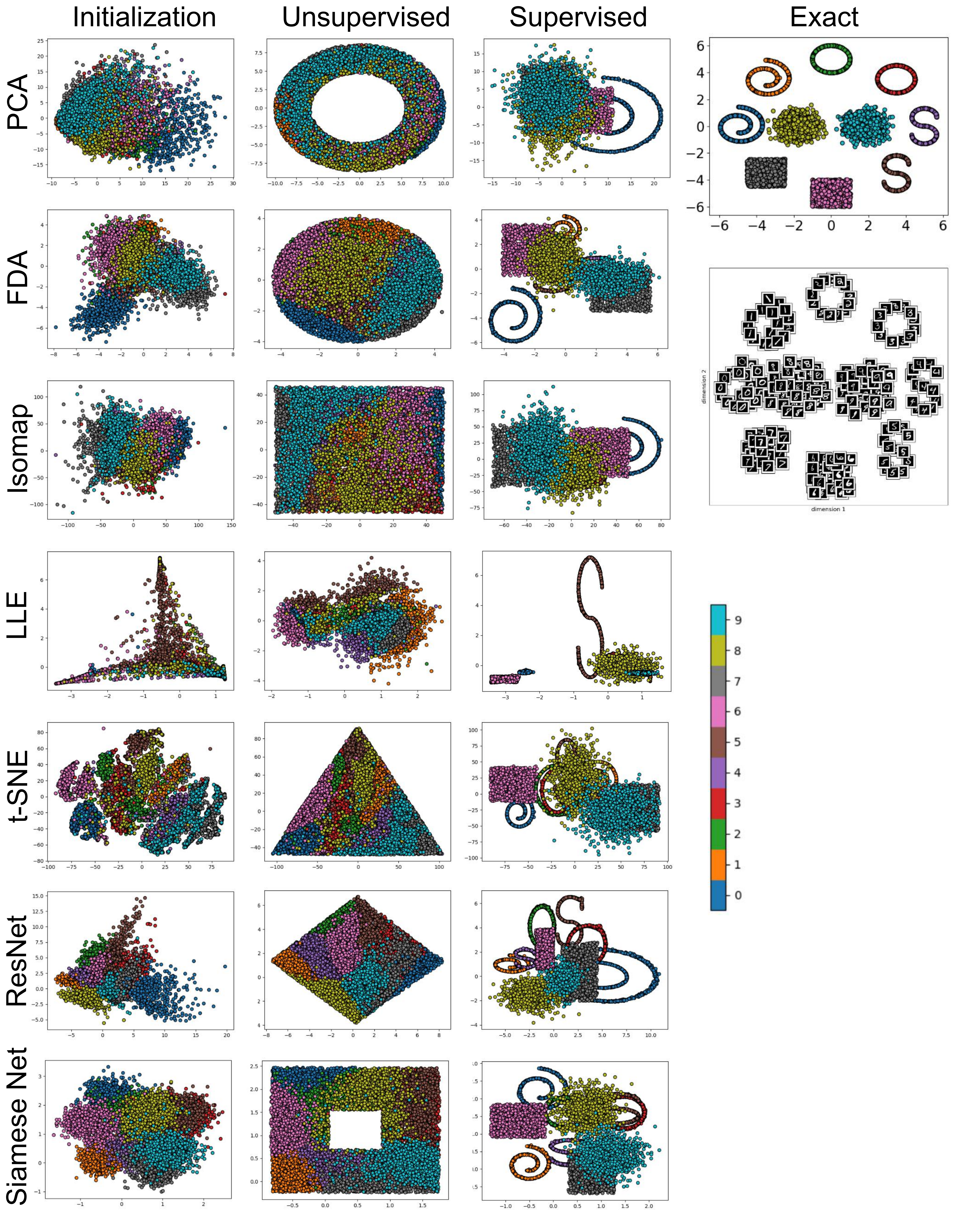}
\caption{Unsupervised and supervised exact manifold embedding of the image data with different initializations. Transformation to exact reference distribution is also shown. The initialization of LLE is scaled by constant to be in range of other embeddings.}
\label{figure_manifoldEmbedding_mnist}
\end{figure*}

\begin{figure*}[!t]
\centering
\includegraphics[width=7in]{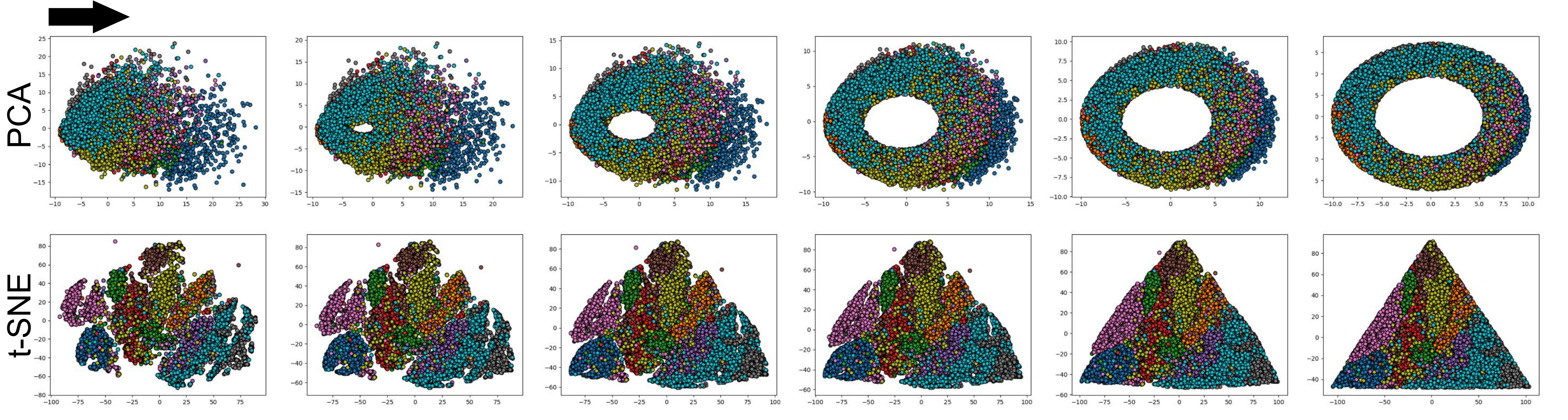}
\caption{Some iterations of unsupervised manifold embedding initialized by PCA and t-SNE. The arrow shows the direction of gradual changes.}
\label{figure_manifoldEmbedding_mnist_iterations}
\end{figure*}

\subsubsection{Standard Reference Distributions}

A simple option for the reference distribution is a standard probability distribution. As an example, we drew a sample of size $1000$ from the two dimensional uniform distribution in range $[0.5, 1.5]$ in both dimensions. This sample is depicted at the right hand side of Fig. \ref{figure_distribution_transform_synthetic2}. We also created an S-shape dataset, with mean zero and in range $[-1.5, 1.5]$ in both dimensions, illustrated at the left hand side of Fig. \ref{figure_distribution_transform_synthetic2}. As this figure shows, in transforming the S-shape data to the shape of uniform distribution, the dataset gradually expands to fill the gaps and become similar to the uniform distribution without changing its mean and scale.  
In transforming to the exact uniform distribution, however, the mean and scale of data change gradually, by translation and contraction, to match the moments of the reference distribution. 
The runtime for this experiment is reported in Table \ref{table_timing}. 
The KL-divergence, MMD, and HSIC of distribution transformation of S-shape data to either the shape of uniform distribution or the exact uniform distribution are reported in Table \ref{table_comparison_distributions}. As expected, after applying QQE, the KL-divergence and MMD have decreased and HSIC has often increased. 
Note that transformations to exact distribution mostly have smaller KL-divergence and MMD and larger HSIC compared to transformations to the shape of reference distribution. This is because exact transformation matches all moments while some moments are not matched in shape transformation. 

\subsubsection{Given Reference Sample}

We can also define the target distribution to transform to using an empirical reference sample. An example is the S-shape data shown in Fig. \ref{figure_distribution_transform_synthetic2} where we transform the uniform data to its distribution. In shape transformation, two gaps appear first to imitate the S shape and then the stems become narrower iteratively. In exact transformation, however, the mean and scale of data also change. Note that exact transformation is harder than shape transformation because of the change of moments; thus, some points jump at initial iterations and then converge gradually. In Section \ref{section_conclusion}, we report on future work to make QQE more robust to these jumps. 
The runtime for this experiment is reported in Table \ref{table_timing}. 
The KL-divergence, MMD, and HSIC of distribution transformation of uniform data to either the shape of S-shape distribution or the exact S-shape distribution are reported in Table \ref{table_comparison_distributions}. As expected, after applying QQE, the KL-divergence has decreased and HSIC has often increased. 
Again, transformation to exact distribution mostly have smaller KL-divergence and MMD and larger HSIC compared to transformation to the shape of reference distribution, for the reason explained before.
The experiments of distribution transformation to a standard or a given distribution show that the proposed QQE can change the distribution of data to any distribution. This desired reference distribution can be a simple or a complicated distribution. Moreover, it can be either a theoretical distribution or an available reference sample. 

\subsubsection{Given Cumulative Distribution Function}

Instead of a standard reference distribution or a reference sample, the user can give a desired CDF for the distribution to have. The reference sample can be sampled using the inverse CDF \cite{ghojogh2020sampling}. The CDF can be multivariate; however, for the sake of visualization, Fig. \ref{figure_CDF_synthetic}-a shows an example multi-modal univariate CDF. We used this CDF and uniform distribution for the first and second dimensions of the reference sample, respectively, shown in  Fig. \ref{figure_CDF_synthetic}-b. QQE was applied on the Gaussian data shown in Fig. \ref{figure_CDF_synthetic}-c and its distribution changed to have a CDF similar to the reference CDF (see Fig. \ref{figure_CDF_synthetic}-d). 
The runtime for this experiment is reported in Table \ref{table_timing}. 
The KL-divergence, MMD, and HSIC of distribution transformation of Gaussian data to either the given CDF are reported in Table \ref{table_comparison_distributions}. As expected, after applying QQE, the KL-divergence and MMD have decreased and HSIC has increased. 
This experiment shows that the proposed QQE gives flexibility to user to even choose the desired distribution by a CDF function or plot. This validates the user-friendliness of QQE algorithm. 

\subsection{Discussion on Impact of Hyperparameters}\label{section_hyperparameter_discussion}

Here, we discuss the impact of hyperparameters $\lambda$, $\eta$, and $k$ on the performance of QQE.
QQE is not yet applicable on out-of-sample data (see Section \ref{section_conclusion}) so these parameters cannot be determined by validation; however, here, we briefly discuss the impact of these hyperparameters. 
For better understanding of discussion, we illustrate the performance of QQE under different hyperparameter settings in Fig. \ref{figure_hyperparameters}. This illustration shows the impact of hyperparameters on distribution transformation or manifold embedding by QQE if the transformations are performed in the input space or embedding space, respectively. 
The average time of quasi-Newton iterations in QQE for experiments of Fig. \ref{figure_hyperparameters} are reported in Table \ref{table_hyperparameter_effect_time}. 

The regularization parameter $\lambda$ determines the importance of distance preserving compared to the quantile-quantile plot of distributions. The larger this parameter gets, the less important the distribution transformation becomes compared to preserving distances; hence, the slower the progress of optimization gets. 
As Fig. \ref{figure_hyperparameters} illustrates, small enough $\lambda$ converges both faster and better.
The value $\lambda=0.1$ was empirically found to be proper for different datasets. 
The learning rate $\eta$ should be set small enough to have progress in optimization without oscillating behavior. 
As shown in Fig. \ref{figure_hyperparameters}, larger $\eta$ makes convergence faster but may result in divergence of optimization.  
We empirically found $\eta=0.01$ or $\eta=0.1$ to be good for different datasets. 
The larger number of neighbors $k$ results in slower pacing of optimization because of Eqs. (\ref{equation_QQE_cost_gradient}) and (\ref{equation_QQE_cost_hessian}). This can be validated by average time of QQE for large value of $k$ reported in Table \ref{table_hyperparameter_effect_time}. Very small $k$, however, does not capture the local patterns of data \cite{saul2003think}. 
For this reason, as Fig. \ref{figure_hyperparameters} depicts, small $k$ does not perform perfectly for QQE.
The value $k=10$ is fairly proper for different datasets.

\begin{table*}[!t]
\caption{The average time of QQE iterations for experiments on the impact of hyperparameters, illustrated in Fig. \ref{figure_hyperparameters}. The reported average times are in seconds.}
\label{table_hyperparameter_effect_time}
\begin{minipage}{\textwidth}
\renewcommand{\arraystretch}{1.3}  
\centering
\scalebox{1}{    
\begin{tabular}{l || c | c | c || c | c | c || c | c | c | c}
\hline
\hline
Transformation type & $\lambda=0.1$ & $\lambda=1$ & $\lambda=10$ & $k=1$ & $k=10$ & $k=50$ & $\eta=0.01$ & $\eta=0.1$ & $\eta=1$ & $\eta=10$ \\
\hline
Shape & 0.63 & 0.56 & 0.61 & 0.47 & 0.63 & 1.11 & 0.63 & 0.51 & 0.49 & 0.55 \\
Exact & 0.61 & 0.20 & 0.35 & 0.35 & 0.61 & 1.03 & 0.61 & 0.53 & 0.53 & 0.48 \\
\hline
\hline
\end{tabular}%
}
\end{minipage}
\end{table*}

\subsection{Distribution Transformation for Image Data}

The distribution transformation can be used for any real data such as images. We divided the ORL facial images \cite{samaria1994parameterisation,web_ORL_dataset} into two sets of with and without eyeglasses. The set with eyeglasses was taken as the reference sample and we transformed the set without glasses to have the shape of reference distribution. Figure \ref{figure_distribution_transform_Glasses} illustrates the gradual change of two example faces from not having eyeglasses to having them. The glasses have appeared gradually in the eye regions of faces. 
The runtime for this experiment is reported in Table \ref{table_timing}. 
The KL-divergence, MMD, and HSIC of distribution transformation of facial image data are reported in Table \ref{table_comparison_distributions}. As expected, after applying QQE, the KL-divergence has decreased.
As shown by this experiment, the proposed QQE can be useful for image processing and image modification purposes where the distribution of image is changed to a desired theoretical distribution or the distribution of another set of images.

\subsection{Manifold Embedding for Synthetic Data}

To test QQE for manifold embedding, we created a three dimensional synthetic dataset having three classes shown in Fig. \ref{figure_manifoldEmbedding_synthetic}. Different dimensionality reduction methods, including PCA \cite{ghojogh2019unsupervised}, FDA \cite{ghojogh2019fisher}, Isomap \cite{tenenbaum2000global}, LLE \cite{roweis2000nonlinear}, and t-SNE \cite{van2009learning}, were used for initialization (see the first column in Fig. \ref{figure_manifoldEmbedding_synthetic}).
There are multiple experiments shown in Fig. \ref{figure_manifoldEmbedding_synthetic} which we explain in the following:
\begin{itemize}
\item For our \textit{unsupervised} experiment, we used a uniform distribution as reference and transformed the entire embedded data in an unsupervised manner. As the second column in Fig. \ref{figure_manifoldEmbedding_synthetic} shows, the embeddings of the entire dataset have changed to have the \textit{shape} of the uniform distribution but the order and adjacency of classes/points differ depending on the initialization method. 
\item The results of our \textit{supervised} experiments are shown in the third column in Fig. \ref{figure_manifoldEmbedding_synthetic}. The desired reference distribution for every class was a uniform distribution and we desired the \textit{shape} of a uniform distribution. As the figure depicts, the supervised QQE has made the shape of distribution of every class uniform without changing its mean and scale.
\item The last column of Fig. \ref{figure_manifoldEmbedding_synthetic} shows the \textit{supervised} transformation of every embedded class to an \textit{exact} reference distribution. The three exact reference distributions (one for each class) are uniform distributions with different means. In exact transformation, the adjacency of points differ depending on the initialization method but the data patterns are similar so we show only one result.
\end{itemize}

The runtime for these experiments are reported in Table \ref{table_timing}. 
The KL-divergence, MMD, and HSIC of manifold embedding by QQE are reported in Table \ref{table_comparison_distributions}. As expected, after applying QQE, the KL-divergence and MMD have often decreased and HSIC has often increased.

\begin{figure*}[!t]
\centering
\includegraphics[width=\textwidth]{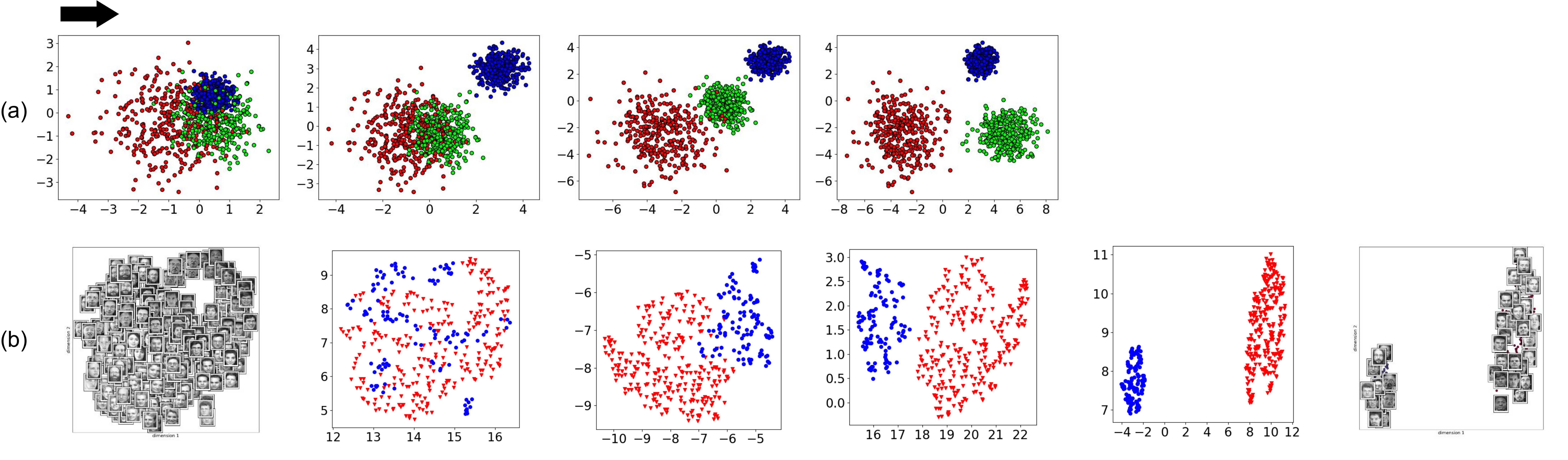}
\caption{Separation and discrimination of classes in synthetic and image data. The arrow shows the direction of gradual changes.}
\label{figure_class_separation}
\end{figure*}

\begin{table}[!t]
\caption{The Recall@k measure for $k \in \{1,2,4,8\}$ for evaluation of separation of classes using QQE algorithm. This table is a quantitative measure of the steps of the experiments shown in Fig. \ref{figure_class_separation}.}
\label{table_recall_at}
\renewcommand{\arraystretch}{1.3}  
\centering
\scalebox{1}{    
\begin{tabular}{l || c | c | c | c}
\hline
\hline
Synthetic Data & Step 1 & Step 2 & Step 3 & Step 4 \\
\hline
Recall@1 & 71.00 & 83.20 & 97.00 & 100.00 \\
Recall@2 & 83.60 & 91.70 & 98.60 & 100.00 \\
Recall@4 & 91.60 & 96.40 & 98.80 & 100.00 \\
Recall@8 & 95.50 & 98.70 & 99.00 & 100.00 \\
\hline
\hline
Face (Eye-glasses) Data & Step 1 & Step 2 & Step 3 & Step 4 \\
\hline
Recall@1 & 89.25 & 97.25 & 100.00 & 100.00 \\
Recall@2 & 95.75 & 98.75 & 100.00 & 100.00 \\
Recall@4 & 98.75 & 99.25 & 100.00 & 100.00 \\
Recall@8 & 99.75 & 99.50 & 100.00 & 100.00 \\
\hline
\hline
\end{tabular}%
}
\end{table}

\subsection{Image Manifold Embedding}

QQE can be used for manifold embedding of real data such as images. For the experiments, we sampled 10000 images from the MNIST digit dataset \cite{lecun1998gradient} with 1000 images per digit. This sampling is because of computational reasons for the time complexity of QQE (see Section \ref{section_conclusion}). We used different initialization methods, i.e., PCA \cite{ghojogh2019unsupervised}, FDA \cite{ghojogh2019fisher}, Isomap \cite{tenenbaum2000global}, LLE \cite{roweis2000nonlinear}, t-SNE \cite{van2009learning}, ResNet-18 features \cite{he2016deep} (with cross entropy loss after the embedding layer), and deep triplet Siamese features \cite{schroff2015facenet} (with ResNet-18 as the backbone network). Any embedding space dimensionality can be used but here, for visualization, we took it to be two. The initialized embeddings are illustrated in the first column in Fig. \ref{figure_manifoldEmbedding_mnist}.

Figure \ref{figure_manifoldEmbedding_mnist} shows the results of experiments which we explain in the following:
\begin{itemize}
\item For \textit{unsupervised} QQE, we took ring stripe, filled circle, uniform (square), Gaussian mixture model, triangle, diamond, and thick square as the reference distribution for embedding initialized by PCA, FDA, Isomap, LLE, t-SNE, ResNet, and Siamese net, respectively. 
As shown in the second column in Fig. \ref{figure_manifoldEmbedding_mnist}, the shape of entire embedding has changed to the desired while the local distances are preserved as much as possible. 
Figure \ref{figure_manifoldEmbedding_mnist_iterations} illustrates some iterations of changes in PCA and t-SNE embeddings as examples. 
\item For \textit{supervised} transformation to the \textit{shape} of references distributions, we used different distributions to show that QQE can use any various references for different classes. Helix, circle, S-shape, uniform, and Gaussian were used for the digits 0/1, 2/3, 4/5, 6/7, 8/9, respectively. The third column in Fig. \ref{figure_manifoldEmbedding_mnist} depicts the supervised transformation to shapes of distributions. 
\item The fourth column in Fig. \ref{figure_manifoldEmbedding_mnist} shows the \textit{supervised} QQE embedding to the \textit{exact} reference distributions. We set the means of reference distributions to be on a global circular pattern. As the fourth column in Fig. \ref{figure_manifoldEmbedding_mnist} shows, it resulted in the transformation of classes to the exact reference distributions on a circular pattern. The images of embedded digits are also shown in this figure.
\end{itemize}

The runtime for these experiments are reported in Table \ref{table_timing}. 
The KL-divergence, MMD, and HSIC of manifold embedding by QQE are reported in Table \ref{table_comparison_distributions}. As expected, after applying QQE, the KL-divergence and MMD have often decreased and HSIC has often increased. 
The experiments of manifold embedding for both synthetic and image data show that the proposed QQE fills the gap of having a manifold learning method with ability to choose the embedding distribution. This is important because the manifold learning methods so far did not give this freedom to user or they forced a specific distribution.

\begin{figure*}[!t]
\centering
\includegraphics[width=\textwidth]{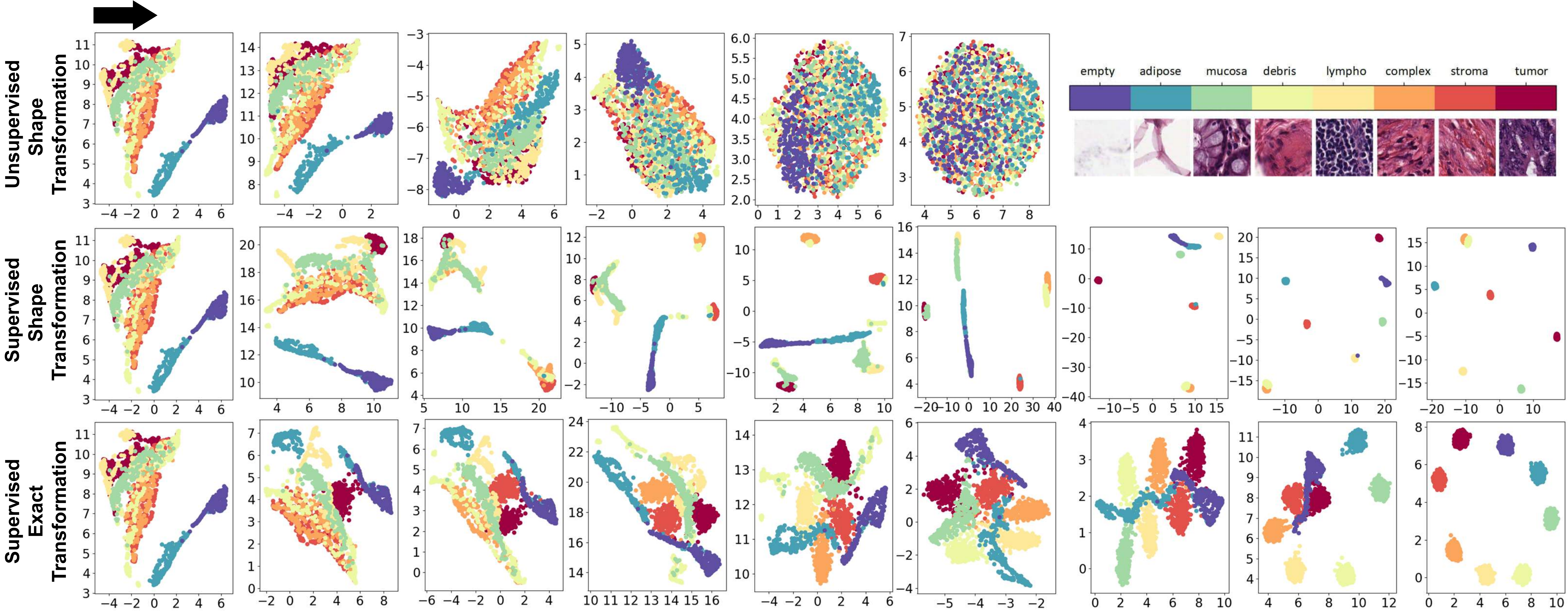}
\caption{Applying QQE algorithm, for unsupervised shape transformation, supervised shape transformation, and supervised exact transformation, on colorectal histopathology embedding. The initialization method for manifold embedding was FDT loss with embedding dimensionality of 128. UMAP is used for 2D visualization of embeddings.} 
\label{figure_Histopathology}
\end{figure*}

\subsection{QQE for Separation of Classes}

QQE can be used for separation and discrimination of classes; although, it does not yet support out-of-sample data (see Section \ref{section_conclusion}). 
For this, reference distributions with far-away means can be chosen where transformation to the exact distribution is used. Hence, the classes move away to match the first moments of reference distributions. 
We experimented this for both synthetic and image data. A two dimensional synthetic dataset with three mixed classes was created as shown in Fig. \ref{figure_class_separation}. The three classes are gradually separated by QQE to match three Gaussian reference distributions with apart means.

For image data, we used the ORL face dataset \cite{samaria1994parameterisation} with two classes of faces with and without eyeglasses. The distribution transformation was performed in the input (pixel) space. The two dimensional embeddings, for visualization in Fig. \ref{figure_class_separation}, were obtained using the Uniform Manifold Approximation and Projection (UMAP) \cite{mcinnes2018umap}. 
The dataset was standardized and the reference distributions were set to be two Gaussian distributions with apart means. 
As the figure shows, the two classes are mixed first but gradually the two classes are completely separated by QQE. 

The runtime for both of these experiments are reported in Table \ref{table_timing}. 
The KL-divergence, MMD, and HSIC of separation of classes for both synthetic and image data by QQE are reported in Table \ref{table_comparison_distributions}. As expected, after applying QQE, the KL-divergence and MMD have often decreased and HSIC has often increased.
Furthermore, the quantitative evaluation of the separation of classes using QQE is reported in Table \ref{table_recall_at} for both synthetic and image data. Following the literature \cite{qian2019softtriple,nguyen2020improved,sikaroudi2020batch}, we used the Recall@k measure for supervised evaluation of embedding in terms of discrimination of classes. As this table shows, QQE has improved the separation of classes by its steps. This improvement in separation of classes can also be seen in Fig. \ref{figure_class_separation}.
This experiment shows that the proposed QQE can be useful for separation and discrimination of classes either in the input space or embedding space. Discrimination of classes helps better classification of data as well as better representation of data in terms of classes.

\subsection{Evaluating QQE for Histopathology Data}\label{section_histopathology}

Although the MNIST and ORL face datasets are also real-world datasets, we evaluated the proposed QQE on medical image data as another real dataset. Finding an informative embedding space for extracting features from medical images is useful for image search and finding similar tumorous image patches in hospital's archives \cite{kalra2020pan}; therefore, it can be used for automatic cancer diagnosis. We used the Colorectal Cancer (CRC) dataset \cite{kather2016multi} which contains ($150 \times 150$)-pixel image patches from colorectal histopathology whole slide images. 
This dataset contains eight tissue types which are background (empty), adipose tissue, mucosal glands, debris, immune cells (lymphoma), complex stroma, simple stroma, and tumor epithelium. Some examples of these tissue types are shown in Fig. \ref{figure_Histopathology}.
We applied QQE manifold embedding on this dataset where Siamese network \cite{schroff2015facenet} with the Fisher Discriminant Triplet (FDT) loss \cite{ghojogh2020fisher} and ResNet-18 backbone \cite{he2016deep} was employed for initialization. 
For this experiment, $\eta=0.1$ was found to be proper without oscillation. 
We set the embedding dimensionality to be 128 to show that QQE also works well on multi-dimensional embedding spaces in addition to two-dimensional spaces.

Figure \ref{figure_Histopathology} illustrates some iterations of applying QQE on the histopathology embedding. We used UMAP \cite{mcinnes2018umap} for 2D visualizations of 128-dimensional embeddings. 
The runtime for these experiments are reported in Table \ref{table_timing}. 
The KL-divergence, MMD, and HSIC of these experiments are reported in Table \ref{table_comparison_distributions}.
For unsupervised shape transformation, we used a multivariate Gaussian reference distribution. As shown in the figure, the UMAP of entire embedding becomes like Gaussian because the embedding is transformed to be Gaussian. 
For supervised shape transformation, we considered a multivariate Gaussian distribution for each tissue type. As expected, the UMAP visualization of embedding shows that each class has a Gaussian form eventually. 
We also performed experiment on supervised transformation to exact reference distribution. For this, we considered eight multivariate Gaussian distributions placed on a global circular pattern similar to what we had in Fig. \ref{figure_manifoldEmbedding_mnist}. The UMAP visualization of transformed embedding validates that the tissue embeddings have been placed on a global circular pattern using QQE. 

The distribution transformation and manifold embedding for histopathology data have various applications and usages. Although deep metric learning has extracted useful features for tissue types, the embedding of patches have not been separated completely. Both supervised exact transformation and supervised shape transformation can be used to separate the tissue types as desired in the embedding space. In the former, the relative locations of tissue types in the embedding space are also desired to be chosen by user. However, in the latter, the distribution of every tissue type is noticed without changing the relative locations of tissue types in the space. 
Therefore, QQE may be used for discriminating tumorous tissues from the normal tissues for better cancer diagnosis.
The unsupervised transformation can also be used to change the distribution of all tissue types together in the embedding space. These transformations and embeddings can be used in hospitals for several reasons. One possible reason may be that the classifier model, to which the embeddings of tissues are going to be fed, requires a specific distribution to work better. Another reason can be the request of doctors and specialists to analyze tissue types in specific distributions. For any reason which requires distribution transformation or manifold embedding with ability to choose the embedding distribution, QQE can be useful in practice.

\section{Conclusion and Future Directions}\label{section_conclusion}

In this paper, we proposed QQE for distribution transformation and manifold embedding. This method can be used for both transforming to the exact reference distribution or its shape. Both unsupervised and supervised versions of this method were also proposed. The proposed method was based on quantile-quantile plot which is usually used in visual statistical tests. 
Experiments were performed on synthetic data, facial images, digit images, and medical histopathology images. We showed that QQE can be used for transforming distribution of data to any desired simple or complicated distribution. The desired distribution can be a theoretical PDF/CDF or an available reference sample. Experiments showed QQE can also be used for modifying images to have a specific distribution or the distribution of another set of images. We also showed that QQE can be used for separation of classes in the input space or embedding space. Experiments on medical data demonstrated that QQE can be useful for practical purposes such as discriminating tumorous tissues from normal ones for better cancer diagnosis. 

There exist several possible future directions. The first future direction is to improve the time complexity of QQE.
Since the complexity of QQE is $\mathcal{O}(n^3)$, dealing with big data would be a challenge for this initial version.
Thus, the immediate future direction for research would be to develop a more sample-efficient approach including handling large datasets. Handling out-of-sample data is another possible future direction. 
Moreover, QQE uses the least squares problem which is not very robust. Because of this, especially if the moments of data and reference distribution differ significantly and we want to transform to the exact reference distribution, some jumps of some data points may happen at initial iterations. This results in later convergence of QQE. One may investigate high breakdown estimators for robust regression \cite{yohai1987high} to make QQE more robust and faster.

\appendices

\section{Proof of Proposition 1}\label{section_appendix_A}

Consider the first part of the cost function:
\begin{align*}
&\mathcal{L}_1 := \frac{1}{2}\, \sum_{i=1}^n \|\b{x}_{i} - \b{y}_{\sigma(i)}\|_2^2 = \frac{1}{2}\, \sum_{i=1}^n \sum_{l=1}^d (x_{i,l} - y_{\sigma(i),l})^2 \\
&\implies \frac{\partial \mathcal{L}_1}{\partial x_{i,l}} = (x_{i,l} - y_{\sigma(i),l}).
\end{align*}
Consider the second part of the cost function:
\begin{align*}
\mathcal{L}_2 := \frac{1}{2a} \sum_{i=1}^n \sum_{j \in \mathcal{N}_i} \frac{\big( d_x(i,j) - d_x^{0}(i,j) \big)^2}{d_x^{0}(i,j)}.
\end{align*}
By chain rule, $\partial \mathcal{L}_2 / \partial x_{i,l} = \partial \mathcal{L}_2 / \partial d_x(i,j) \times \partial d_x(i,j) / \partial x_{i,l}$. 
The first derivative is:
\begin{align*}
\frac{\partial \mathcal{L}_2}{\partial d_x(i,j)} = \frac{1}{a} \sum_{j \in \mathcal{N}_i} \frac{d_x(i,j) - d_x^{0}(i,j)}{d_x^{0}(i,j)},
\end{align*}
and using the chain rule, the second derivative is $\partial d_x(i,j) / \partial x_{i,l} = \partial d_x(i,j) / \partial d_x^2(i,j) \times \partial d_x^2(i,j) / \partial x_{i,l}$. 
We have:
\begin{align}
& \frac{\partial d_x(i,j)}{\partial d_x^2(i,j)} = 1 / \frac{\partial d_x^2(i,j)}{\partial d_x(i,j)} = 1 / (2\, d_x(i,j)). \nonumber \\
& d_x^2(i,j) = \|\b{x}_i - \b{x}_j\|_2^2 = \sum_{k=1}^p (x_{i,l} - x_{j,l})^2. \nonumber \\
& \frac{\partial d_x^2(i,j)}{\partial x_{i,l}} = 2\,(x_{i,l} - x_{j,l}), \nonumber \\
& \therefore ~~~~ \frac{\partial d_x(i,j)}{\partial x_{i,l}} = \frac{x_{i,l} - x_{j,l}}{d_x(i,j)}. \label{equation_Sammon_derivative_dY}
\end{align}
\begin{align*}
&\therefore ~~~ \frac{\partial \mathcal{L}_2}{\partial x_{i,l}} = \frac{1}{a} \sum_{j \in \mathcal{N}_i} \frac{d_x(i,j) - d_x^{0}(i,j)}{d_x(i,j)\, d_x^{0}(i,j)} (x_{i,l} - x_{j,l}).
\end{align*}
Considering both parts of the cost function, the gradient is as in the proposition. Q.E.D.

\section{Proof of Proposition 2}\label{section_appendix_B}

The second derivative is the derivative of the first derivative, i.e., Eq. (\ref{equation_QQE_cost_gradient}). Hence:
\begin{align*}
&\frac{\partial^2 \mathcal{L}}{\partial x_{i,l}^2} = 1 + \frac{\lambda}{a} \sum_{j \in \mathcal{N}_i} \frac{\partial}{\partial x_{i,l}} \Big( \frac{d_x(i,j) - d_x^{0}(i,j)}{d_x(i,j)\, d_x^{0}(i,j)} (x_{i,l} - x_{j,l}) \Big).
\end{align*}
\begin{align*}
&\frac{\partial}{\partial x_{i,l}} \Big( \frac{d_x(i,j) - d_x^{0}(i,j)}{d_x(i,j)\, d_x^{0}(i,j)} (x_{i,l} - x_{j,l}) \Big) \\
&~~~~~~~~ = (x_{i,l} - x_{j,l}) \frac{\partial}{\partial x_{i,l}} \Big( \frac{d_x(i,j) - d_x^{0}(i,j)}{d_x(i,j)\, d_x^{0}(i,j)} \Big) \\
&~~~~~~~~~~~ + \frac{d_x(i,j) - d_x^{0}(i,j)}{d_x(i,j)\, d_x^{0}(i,j)} \underbrace{\frac{\partial}{\partial x_{i,l}} (x_{i,l} - x_{j,l})}_{=1}.
\end{align*}
\begin{align*}
&\frac{\partial}{\partial x_{i,l}} \Big( \frac{d_x(i,j) - d_x^{0}(i,j)}{d_x(i,j)\, d_x^{0}(i,j)} \Big) = \frac{1}{d_x^{0}(i,j)} \frac{\partial}{\partial x_{i,l}} \Big( 1 - \frac{d_x^{0}(i,j)}{d_x(i,j)} \Big) \\
&~~~~~ = \frac{1}{d_x^{0}(i,j)} \underbrace{\frac{\partial}{\partial x_{i,l}} (1)}_{=0} - \underbrace{\frac{d_x^{0}(i,j)}{d_x^{0}(i,j)}}_{=1} \frac{\partial}{\partial x_{i,l}} \Big( \frac{1}{d_x(i,j)}\Big) \\
&~~~~~ = \frac{1}{d_x^2(i,j)} \frac{\partial}{\partial x_{i,l}} ( d_x(i,j) ) \overset{(\ref{equation_Sammon_derivative_dY})}{=} \frac{(x_{i,l} - x_{j,l})}{d_x^3(i,j)}.
\end{align*}
Putting all parts of derivative together gives the second derivative. Q.E.D.

\ifCLASSOPTIONcaptionsoff
  \newpage
\fi



%



\bibliographystyle{IEEEtran}
\bibliography{references}

\end{document}